\documentclass[]{article}

\usepackage{amssymb}
\usepackage{amsmath}
\usepackage{amsthm}
\usepackage{booktabs}
\usepackage{float}
\usepackage{epstopdf}
\usepackage{verbatim}
\usepackage{empheq}
\usepackage{hyperref}
\usepackage{color} 

\usepackage[us, 12hr]{datetime} 
\usepackage{url}
\usepackage[vmargin = 0.9in, hmargin = 0.97in]{geometry}
\usepackage{setspace}
\newcommand{\ds}{\displaystyle}

\newcommand{\mygets}{=}
\newcommand{\real} {\mathbb{R}}
\newcommand{\nreal}{\mathbb{R}_{+}}
\newcommand{\tp}{\mathsf{T}}							% transpose

\newcommand{\lam}{\lambda}

\newcommand{\eps}{\epsilon}

\newcommand{\phie}{\phi_{\epsilon}}

\renewcommand{\AA}{\mathbf{A}}
\newcommand{\BB}{\mathbf{B}}

\newcommand{\DD}{\mathbf{D}}

\newcommand{\GG}{\mathbf{G}}
\newcommand{\HH}{\mathbf{H}}
\newcommand{\II}{\mathbf{I}}

\newcommand{\QQ}{\mathbf{Q}}
\newcommand{\RR}{\mathbf{R}}

\renewcommand{\b}{\mathbf{b}}

\renewcommand{\d}{\mathbf{d}}

\newcommand{\f}{\mathbf{f}}

\newcommand{\h}{\mathbf{h}}
\newcommand{\p}{\mathbf{p}}
\newcommand{\q}{\mathbf{q}}

\renewcommand{\u}{\mathbf{u}}
\renewcommand{\v}{\mathbf{v}}

\newcommand{\w}{\mathbf{w}}
\newcommand{\x}{\mathbf{x}}
\newcommand{\y}{\mathbf{y}}
\newcommand{\z}{\mathbf{z}}

\renewcommand{\le}{ \leqslant }
\renewcommand{\ge}{ \geqslant }

\newcommand{\smooth}{\rho}
\newcommand{\LAM}{\mathbf{\Lambda}}

\newcommand{\figurescale}		{0.45}				% figure
\newcommand{\figurescalesmall}	{0.37}				% small figure
\newcommand{\opt}{^{*}}

\providecommand{\norm}[1]{\lVert#1\rVert}
\providecommand{\abs}[1]{\vert #1 \vert}
\providecommand{\mtxLam}[1]{\big[ \mathbf{\Lambda}(#1) \big]}

\providecommand{\secref}[1]{Section~\ref{#1}}

\providecommand{\figref}[1]{Figure~\ref{#1}}
\providecommand{\eqnref}[1]{\eqref{#1}}

\DeclareMathOperator{\LPFTVD}	{ \text{LPF/TVD} }
\DeclareMathOperator{\name}		{ \text{ETEA} }

\newtheorem{proposition}{Proposition}

\allowdisplaybreaks

\title{Sparsity-based Correction of Exponential Artifacts%
\thanks{Preprint submitted to Signal Processing (Manuscript)}%
\thanks{Corresponding author: Yin Ding (yd372@nyu.edu)}%
}
\author{Yin Ding and Ivan W. Selesnick\\
New York University School of Engineering\\ 6 MetroTech Center, Brooklyn, NY 11201, USA\\
}
\date{}

\renewcommand\footnotemark{}

\begin{document}
\maketitle

%------------------------------------------------------------------------------%
\begin{abstract}
This paper describes an exponential transient excision algorithm (ETEA).
In biomedical time series analysis, e.g., \textit{in vivo} neural recording and electrocorticography (ECoG), 
some measurement artifacts take the form of piecewise exponential transients.
The proposed method is formulated as an unconstrained convex optimization problem, regularized by smoothed $\ell_1$-norm penalty function, 
which can be solved by majorization-minimization (MM) method.
With a slight modification of the regularizer, 
ETEA can also suppress more irregular piecewise smooth artifacts, especially, ocular artifacts (OA) in electroencephalography (EEG) data.
Examples of synthetic signal, EEG data, and ECoG data are presented to illustrate the proposed algorithms.
\end{abstract}

%------------------------------------------------------------------------------%
%\begin{keyword}
%Sparsity,
%signal decomposition, 
%artifact removal.
%\end{keyword}
%------------------------------------------------------------------------------%
%\end{frontmatter}

%------------------------------------------------------------------------------%
\section{Introduction}

This work is motivated by the problem of suppressing various types of artifacts in recordings of neural activity.
In a recent study \cite{artifacts_Islam_2014}, typical artifacts in \textit{in vivo} neural recordings are classified into four types 
(Type 0 to 3, see  Section 2.2 and Figure~3 in \cite{artifacts_Islam_2014}).
This classification covers many artifacts in the scope of human brain activity recordings, 
e.g., electroencephalography (EEG) and electrocorticography (ECoG).
In this paper, we  consider the suppression of Type 0 and Type 1 artifacts.
For the purpose of flexibility and generality, we redefine them in terms of morphological characteristics:
\begin{itemize}
	\item
		Type 0: a smooth protuberance that can be modeled as $\hat x(t) = t e^{-\alpha t}$, when $t \ge t_0$.
	\item 
		Type 1: an abrupt jump followed by an exponential decay that can be modeled as $\hat x(t) = e^{-\alpha t}$, when $t \ge t_0$.
\end{itemize}
\figref{fig:etea_Example_0_type} shows examples of the two types of artifacts.
We do not consider the other two types in this work
because our previous works have addressed efficient algorithms to remove such artifacts.
For instance, low-pass filtering/total variation denoising ($\LPFTVD$) \cite{Selesnick_2013_lpftvd} suppresses  Type 2 artifacts (\figref{fig:etea_Example_0_type}c),
and lowpass filtering/compound sparse denoising (LPF/CSD) \cite{Selesnick_tara_2014, Selesnick_2013_lpftvd} can remove sparse and blocky spikes 
(Type 3 shown in \figref{fig:etea_Example_0_type}d).
 
%------------------------------------------------------------------------------%
\begin{figure}[t]
\centering
	\includegraphics[scale = \figurescalesmall] {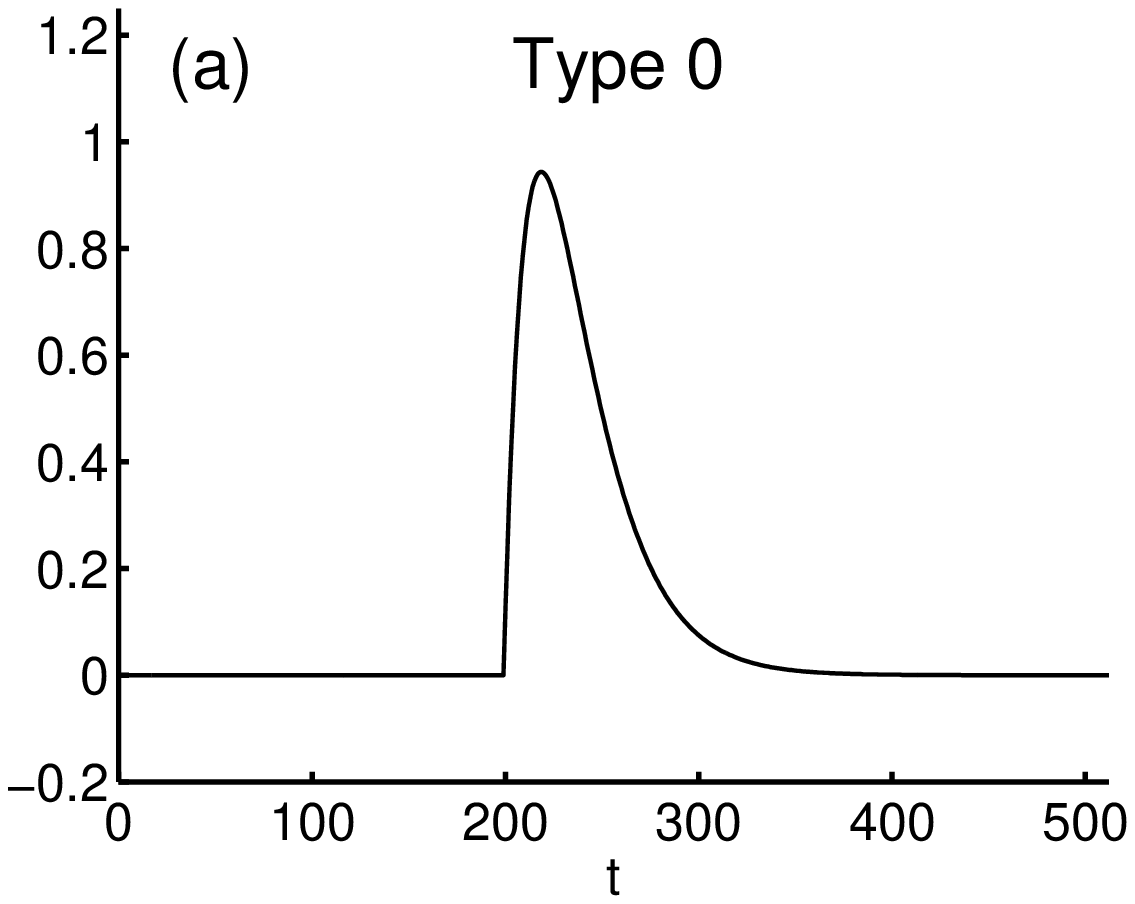}	
	\includegraphics[scale = \figurescalesmall] {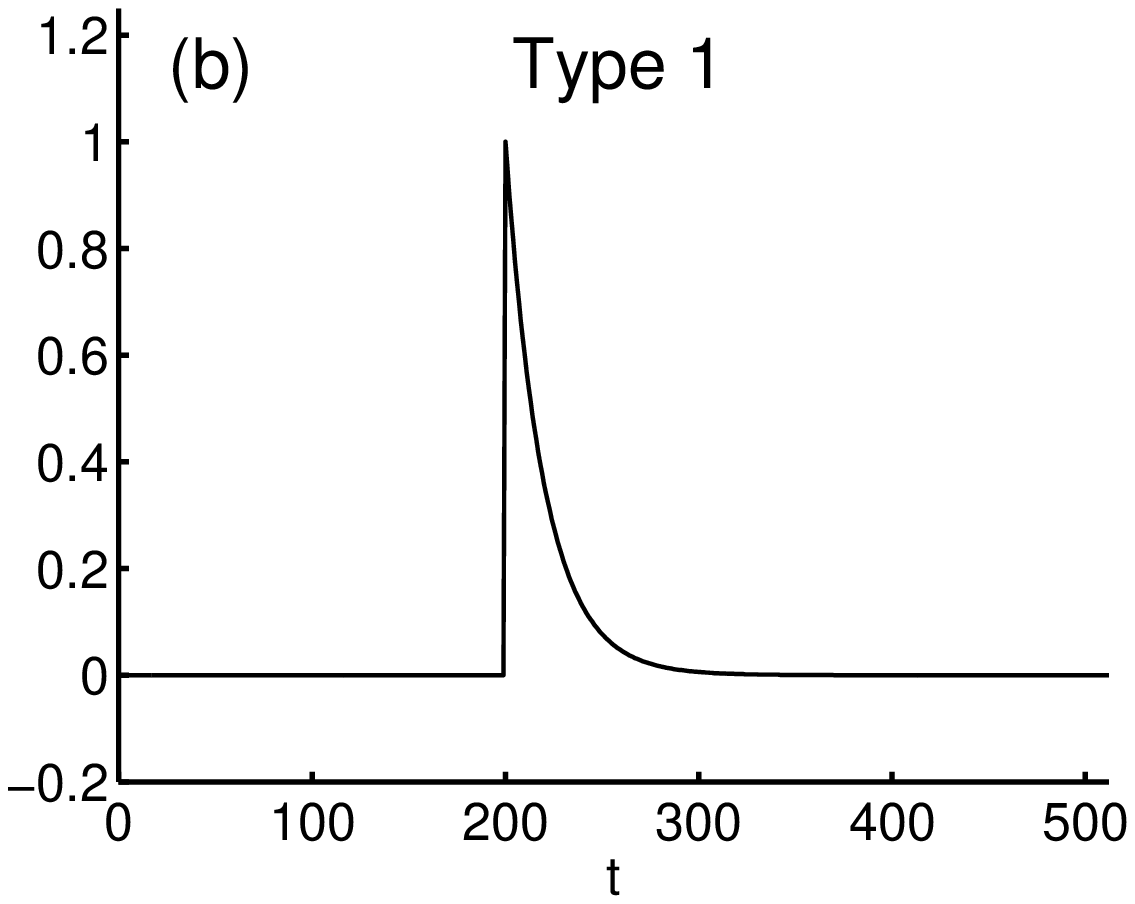}
	\\
	\includegraphics[scale = \figurescalesmall] {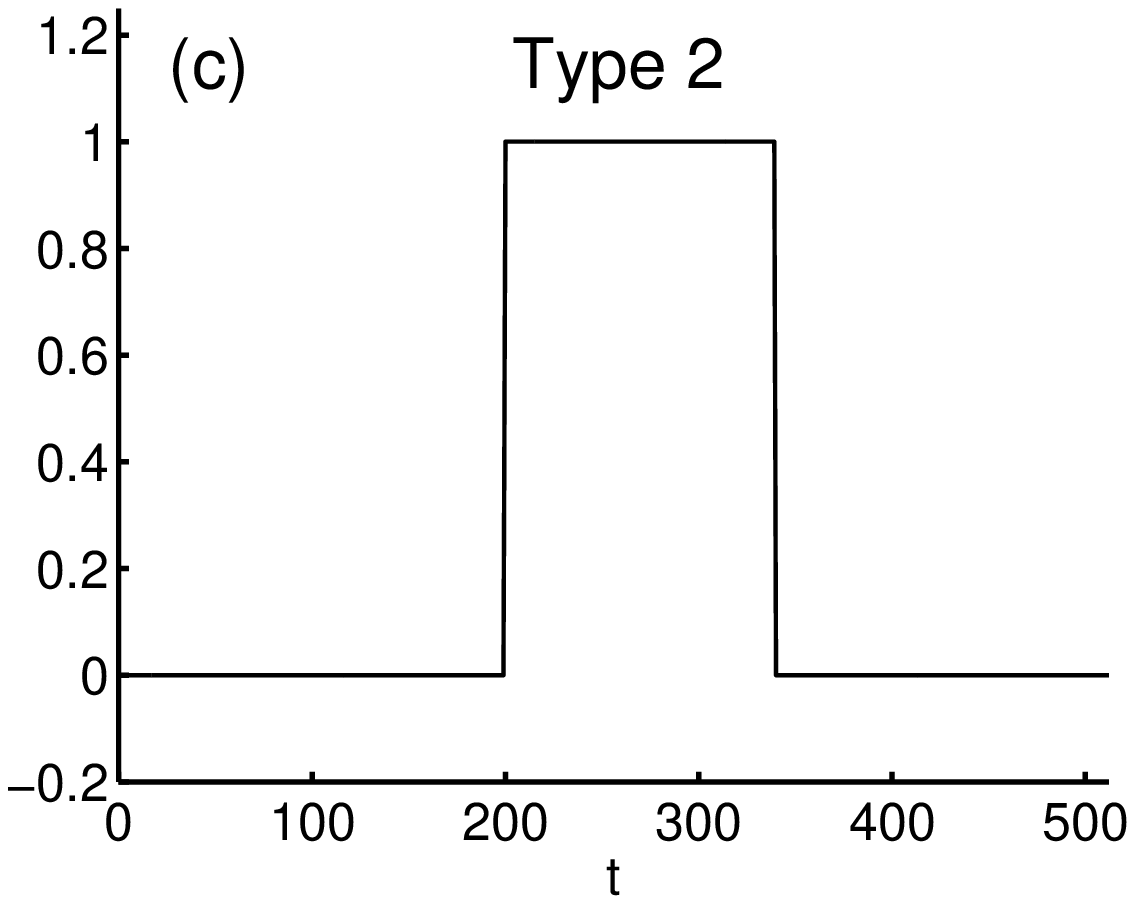}
	\includegraphics[scale = \figurescalesmall] {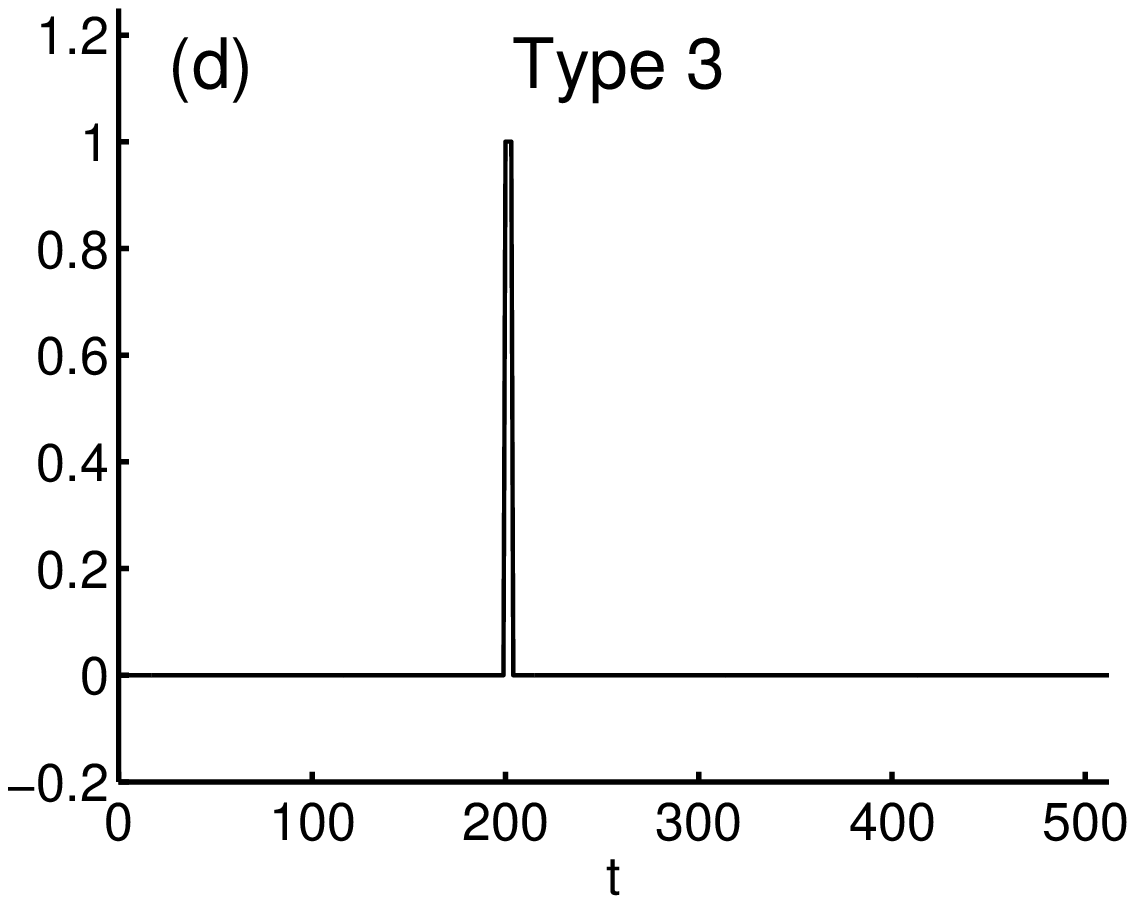}	
	\caption{Examples of
		(a) Type 0 artifact, and
		(b) Type 1 artifact, and
		(c) Type 2 artifact, and
		(d) Type 3 artifact.	}
	\label{fig:etea_Example_0_type}
\end{figure}
%------------------------------------------------------------------------------%

The approach proposed in this paper is based on an optimization problem intended to capture the primary morphological
characteristics of the artifacts using sparsity-inducing regularization.
To formulate the problem, we model the observed time series as
%------------------------------------------------------------------------------%
\begin{align}\label{eqn:etea_model}
	y(t) = f(t) + x(t) + w(t),	
\end{align}
%------------------------------------------------------------------------------%
where $f$ is a lowpass signal, $x$ is a piecewise smooth transient signal (i.e., Type 0 or Type 1 artifacts), and $w$ is stationary white Gaussian noise.
More specifically, $f$ is assumed to be restricted to a certain range of low frequencies.
In other words, $\HH(f) \approx \mathbf{0}$, where $\HH$ is a high-pass filter.
Note that in the signal model \eqnref{eqn:etea_model}, conventional LTI filtering is not suitable to estimate either $f$ or $x$ from $y$,
because component $x$, as a piecewise smooth signal comprised of transients, is not band limited.

In order to estimate the components, we combine LTI filtering with sparsity-based techniques.
We formulate an optimization problem for both decomposition and denoising.
A computationally efficient algorithm is derived to solve the optimization problem, based on the theory of majorization-minimization (MM) 
\cite{FBDN_2007_TIP, mm_Hunter_tutorial_2004, mm_Lange_2000}.

In addition, this paper specifies how to generate a smoothed penalty function and its majorizer from a non-smooth one, 
in order to overcome a numerical issue that arises when the penalty function is not differentiable.

%------------------------------------------------------------------------------%
\begin{table*}[t]
\begin{center}
	\caption{Sparsity-promoting penalty functions ($a>0$).}\medskip
	\scalebox{0.850}{
		\begin{tabular}{@{} c l l l l@{}}
		\toprule
		\textbf{Penalty}	 &$\phi(u)$ 	&$\phie(u)$  &$\psi(u) = u / \phie'(u)\quad $ \\[0.4em]
		\midrule 
			abs	 	& $\abs{u}$						
	 				& $\sqrt{ u^2 + \eps }$		
	 				& $\sqrt{ u^2 + \eps }$ \\[0.8em]	
			log	 	&$\ds\frac{1}{a} \log ( 1 + a |u| )$ 
	 				&$\ds\frac{1}{a} \log ( 1 + a \sqrt{u^2+\eps} )$ 
	 				&$\ds\sqrt{ u^2 + \eps } \left( 1 + a \sqrt{ u^2 + \eps } \right)$	\\[0.8em]
			atan 	&$\ds\frac{2}{a \sqrt{3}} \left( \tan^{-1} \left( \frac{1+2 a |u|} {\sqrt{3}} \right) - \frac{\pi}{6}\right)$
	 				&$\ds\frac{2}{a \sqrt{3}} \left( \tan^{-1} \left( \frac{1+2 a \sqrt{u^2+\eps} } {\sqrt{3}} \right) - \frac{\pi}{6} \right)$
	 				&$\ds\sqrt{ u^2 + \eps } \left( 1 + a \sqrt{ u^2 + \eps } + a^2 (u^2 + \eps) \right)$	\\[0.8em]
		\bottomrule
		\end{tabular}
	}
	\end{center}
	\label{tab:etea_penalty}
\end{table*}
%------------------------------------------------------------------------------%

%------------------------------------------------------------------------------%
\subsection{Related works}
%------------------------------------------------------------------------------%
Some recent works recover signals with transients by various algorithms.
In \cite{Gholami_2013_SP}, a slowly varying signal is modeled as a local polynomial and an optimization problem using Tikhonov regularization is formulated to capture it.
In \cite{Ning_QRS_2013}, the slowly varying trend is modeled as a higher-order sparse-derivative signal (e.g., the third-order derivative is sparse).

Instead of estimating the slowly varying component via regularization, the $\LPFTVD$ method \cite{Selesnick_2013_lpftvd} 
estimates a lowpass component by LTI filtering and a piecewise constant component by optimization.
In this case, an optimization problem is formulated to estimate the piecewise constant component.
The approach proposed here uses a similar technique to recover the lowpass component,
but in contrast to $\LPFTVD$, it is more general
--- the regularization is more flexible with a tunable parameter, so that $\LPFTVD$ can be considered as a special case.

Another algorithm related to the approach taken in this paper is the transient artifact reduction algorithm (TARA) \cite{Selesnick_tara_2014}
which is utilized to suppress additive piecewise constant artifacts and spikes (similar to a hybrid of Type~2 and Type~3 artifact).
The approaches proposed in this work target different types of artifacts (Type~0 and Type~1)
and applied in different applications.

The irregularity of Type 0 transients leads to a more complicated artifact removal problem, where the artifact are irregular fluctuations.
A typical example in EEG is ocular artifacts (OA) caused by the blink and/or movement of eyes.
To suppress OA, there are approaches based on empirical mode decomposition (EMD)
\cite{Molla_2012, eyeblink_Molla_picassp_2012, eyeblink_Molla_picassp_2010, Zeng_2013_Sensors},
and on independent component analysis (ICA) methods 
\cite{Akhtar_2012_SP, eyeblink_Dammers_tbme_2008, eyeblink_GuerreroMosquera_2012, eyeblink_Joyce_2004, Mammone_2012, eyeblink_Noureddin_tbme_2012}.
The concept of spatial-frequency in acoustic analysis is also used to remove OA from multichannel signals
\cite{eyeblink_Nazarpour_tbme_2008, eyeblink_Wongsawat_2008}.
In this work, we present a new method to suppress ocular artifacts by proposing a specific model and using sparse optimization.

This paper adopts a regularizer inspired by the generalized 1-D total variation \cite{tv_Karahanoglu_2011},
wherein the derivative operator in conventional total variation regularizer is generalized to a recursive filter.
The regularizers adopted in $\name$ and second-order $\name$
coincide with first-order and second-order cases of generalized 1-D total variation, respectively.
Some differences to the problem discussed in \cite{tv_Karahanoglu_2011} are as follow.
Firstly, the signal model \eqnref{eqn:etea_model} allows a lowpass baseline as a component,
hence, $\name$ can be seen as a combination of conventional LTI filtering and generalized 1-D total variation. 
Secondly, we consider a formulation in terms of banded matrices, for computational efficiency.
Thirdly, we give optimality conditions of the proposed problems, 
and use these conditions as a guide to set the regularization parameters.

%------------------------------------------------------------------------------%
\section{Preliminaries}

%------------------------------------------------------------------------------%
\subsection{Notation}
%------------------------------------------------------------------------------%
We use bold uppercase letters for matrices, e.g., $\AA$ and $\BB$, and bold lowercase letters for vectors, e.g., $\x$ and $\y$.
We use column vectors for one-dimensional series.
For example, a vector $\x \in \real^{N}$ is written as
%------------------------------------------------------------------------------%
\begin{align}
	\x = \big[ x(0)  , \ x(1)  , \ \cdots , \ x(N-1)  \big]^{\tp}
\end{align}
%------------------------------------------------------------------------------%
where $[ \ \cdot \ ]^{\tp}$ denotes matrix transpose.
The $\ell_1$-norm and squared $\ell_2$-norm of $\x$ are defined as
\begin{align}
	\ds \norm{\x}_1      : =  \sum_{n} \abs{x(n)},
	\qquad
	\ds \norm{\x}_2^{2}  : =  \sum_{n} {x^2(n)}.
\end{align}
The inverse transpose of a matrix is denoted as $\AA^{-\tp}$.
The first-order difference operator $\DD$ of size $(N-1) \times N$ is
%------------------------------------------------------------------------------%
\begin{equation} 
	\DD := 
	\begin{bmatrix}
		-1	&1 		&		& 		&		\\
	   	&  	-1		&1		& 		&  		\\
		&	& 		\ddots  &\ddots	&		\\
		&	&		&   	-1		&1
	\end{bmatrix}.
\label{eqn:etea_D}
\end{equation}
%------------------------------------------------------------------------------%

We use $f(x ; a)$ to denote a function of $x$ determined by parameter $a$, 
to distinguish it from a function with two ordered arguments, e.g., $f(x,y)$.

%----------------------------------------------------%
\subsection{LTI filter with matrix formulation} \label{sec:etea_lti_filter}
%----------------------------------------------------%

A discrete-time filter can be described as
\begin{align} \label{eqn:etea_ARMA}
	\sum_{k} a_k ~y(n-k) = \sum_{k} b_k ~x(n-k),	
\end{align}
with transfer function $H(z) = B(z) / A(z)$.
For a finite length signal, \eqnref{eqn:etea_ARMA} can be rewritten as
\begin{align} \label{eqn:etea_AyBx}
	\AA \y = \BB \x,	
\end{align}
where $\AA, \BB \in \real^{N \times N}$ are both banded matrices.
When $\HH$ is a zero-phase filter with $a_{-k} = a_{k}$ and $b_{-k} = b_{k}$, then $\AA$ and $\BB$ are both symmetric Toeplitz matrices
\begin{subequations} \label{eqn:AB}
\begin{align}	 \label{eqn:A}
	\AA  = 
	\begin{bmatrix}
		a_1		&a_0 		&			& 					\\
	   	a_0	    &a_1			&a_0 		&  				\\
		&		\ddots		&\ddots	&\ddots					\\
		&		&		a_0		&a_1		&a_0			\\
		&		& 		&		a_0			&a_1								
	\end{bmatrix},
	\\		
	\label{eqn:B}		
	\BB  = 
	\begin{bmatrix}
			b_1 		&b_0			& 						\\
			b_0		&b_1 		&b_0							\\
			&		\ddots	&\ddots	&\ddots						\\
			&		&		b_0		&b_1 		&b_0			\\
		 	&		&		&		b_0		&b_1 		
	\end{bmatrix}.			
\end{align}
\end{subequations}
In this case, disregarding a few samples on both edges, $\y$ can be re-written as:
\begin{align}\label{eqn:etea_banded_filter}
	\y  = \BB \AA^{-1} \x = \HH \x.	
\end{align}
A detailed description of such zero-phase filters is given in Section~V of Ref.~\cite{Selesnick_2013_lpftvd},
where two parameters, $d$ and $f_c$, determine the order of the filter and the cut-off frequency, respectively.
In this paper, $\BB$ is a square matrix, in contrast to~\cite{Selesnick_2013_lpftvd}. The simplest case ($d=1$) is given by~\eqnref{eqn:AB}.

%------------------------------------------------------------------------------%
\subsection{Majorization-Minimization}\label{sec:etea_majorization_minimization}
%------------------------------------------------------------------------------%

Majorization-Minimization (MM) \cite{FBDN_2007_TIP, opt_Lange_book_2004, mm_Schifano_2010}
is a procedure to replace a difficult optimization problem by a sequence of simpler ones
\cite{FBDN_2007_TIP, mm_Hunter_tutorial_2004, mm_Lange_2000, mm_Schifano_2010}.
%------------------------------------------------------------------------------%
Suppose a minimization problem has a objective function $J : \real^N  \to \real$,
then the majorizer $G( \u , \v ) : \mathbb{R}^N \times \mathbb{R}^N \to \mathbb{R}$ satisfies
\begin{subequations} \label{eqn:etea_mm_condition}
\begin{align} 
	G( \u , \v ) & =   J(\u), \quad \text{when} \ \u = 	 \v, 	  \label{eqn:etea_mm_condition_a}\\
	G( \u , \v ) & \ge J(\u), \quad \text{when} \ \u \neq \v.     \label{eqn:etea_mm_condition_b}
\end{align}
\end{subequations}
%------------------------------------------------------------------------------%
Then MM iteratively solves
\begin{align} \label{eqn:etea_mm_iteration}
	\u^{(k+1)} =  \arg \min_{\u} G(\u , \u^{(k)}) 	
\end{align}
until convergence, where $k$ is iteration index. 
The detailed derivation and proof of convergence of MM are given in \cite{opt_Lange_book_2004}.

%------------------------------------------------------------------------------%
\section{Majorization of smoothed penalty function}
%------------------------------------------------------------------------------%
\begin{figure}[t]
\centering
	\includegraphics[scale = \figurescalesmall] {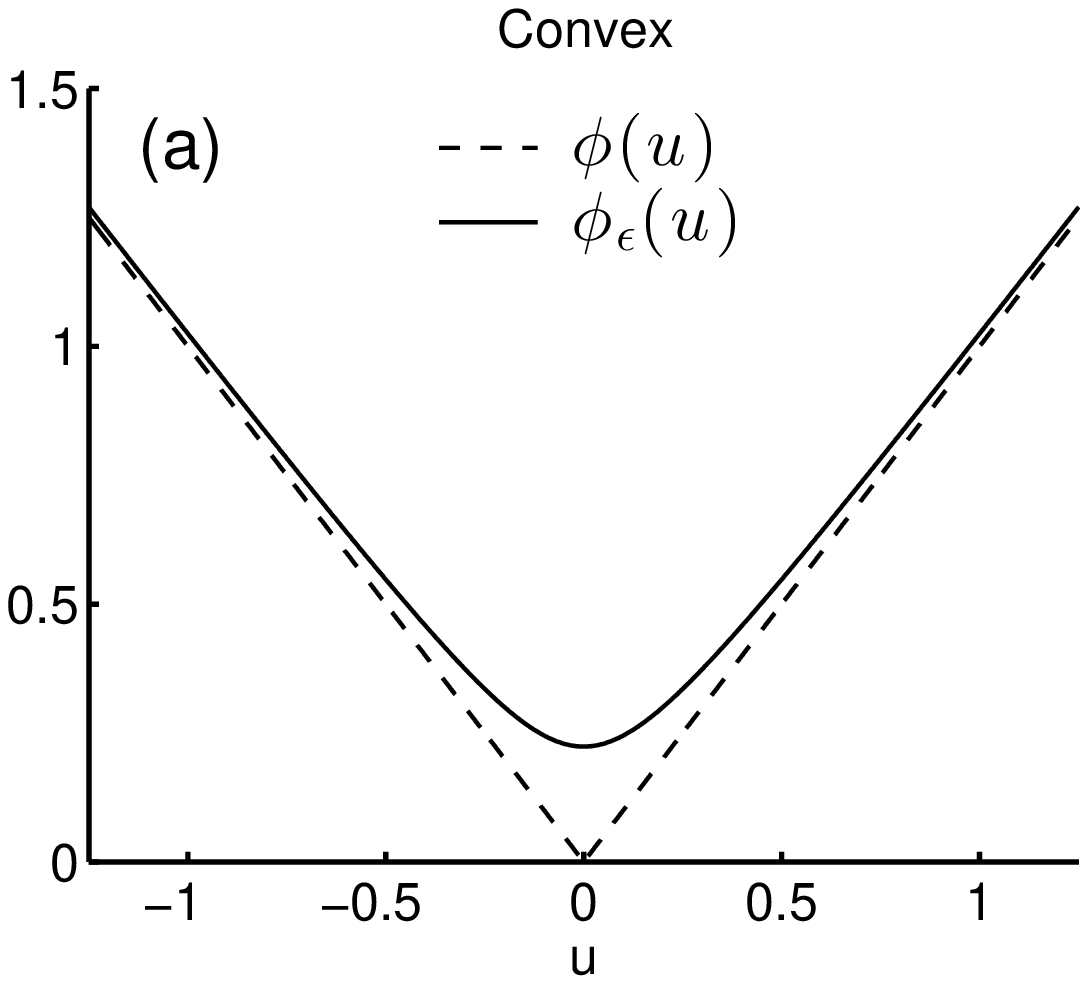}	
	\includegraphics[scale = \figurescalesmall] {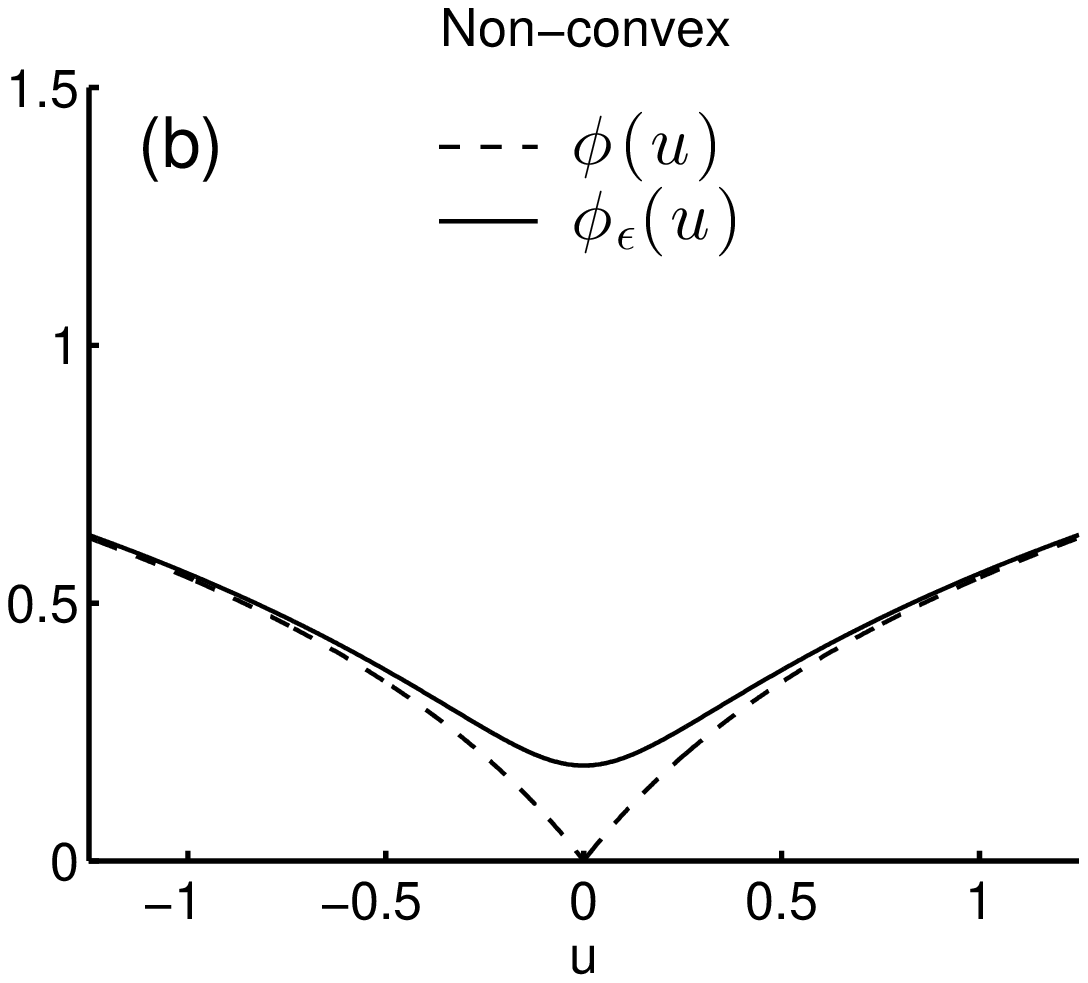}
	\\	
	\includegraphics[scale = \figurescalesmall] {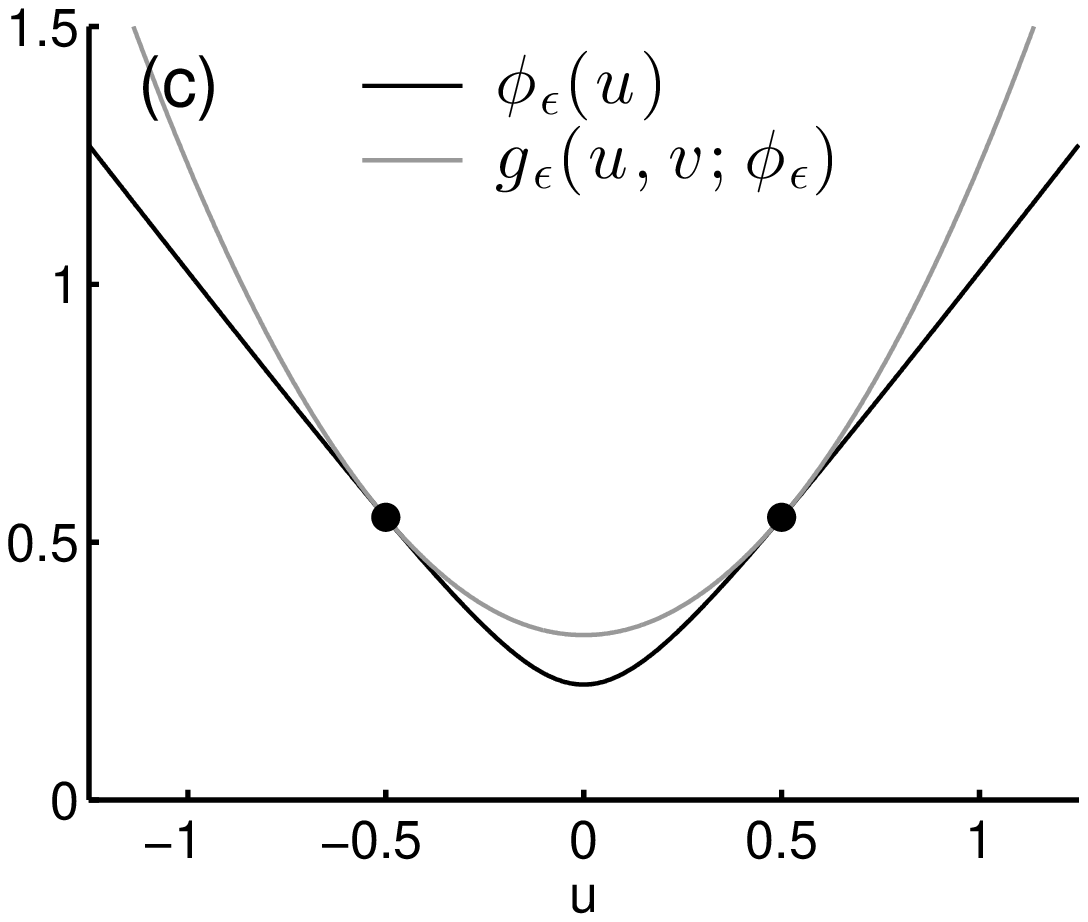}	
	\includegraphics[scale = \figurescalesmall] {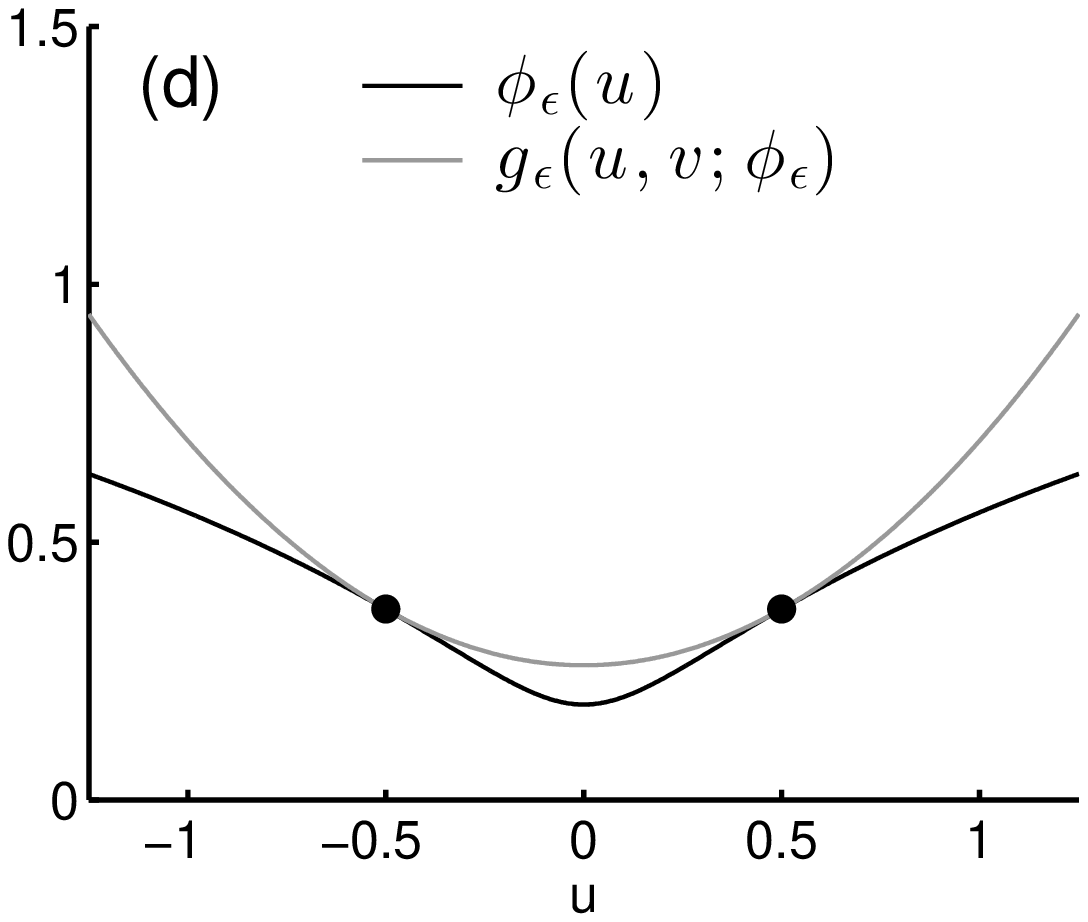}		
	\caption{%
		(a) Convex non-smooth penalty function and corresponding smoothed penalty function.
		(b) Non-convex  and smoothed penalty functions.
		(c) Majorizer (gray) of the smoothed penalty function in (a).
		(d) Majorizer (gray) of the smoothed non-convex penalty function in (b).
		For all the figures we set $\eps = 0.05, v = 0.5, a = 2$.	}
	\label{fig:etea_Example_0_penalty}
\end{figure}
%------------------------------------------------------------------------------%

When using sparse-regularized optimization to recover a signal, the $\ell_1$-norm is widely utilized.
To further enhance sparsity, some methods use iteratively re-weighting procedures \cite{Candes_2008_JFAP, Kozlov_2010_geomath, Needell_reweight_2009, Wipf_2010_TSP},
or a non-convex pseudo $\ell_p$-norm ($ 0 < p < 1 $), 
or a mixed-norm in the regularizer \cite{Cetin_tip_2001, Yu_sp_2012, Yaghoobi_2009_TSP, Selesnick_2012_polynomial, Kowalski_2009}.
Other non-convex penalty functions can also be used.
In \cite{smooth_Nikolova_2005_MMS} a logarithm penalty function is discussed (`log' function in Table~\ref{tab:etea_penalty}),
and in \cite{Selesnick_2013_max_sparse}, an arctangent function is used (`atan' function Table~\ref{tab:etea_penalty}).
However, each of these functions is non-differentiable at zero, where $\phi'(0^{-}) \neq \phi'(0^{+})$.
A method, to avoid the discontinuity in the derivative, is to smooth the function.

Consider the smoothed function $\phi_{\eps} : \real \to \nreal$
\begin{align} \label{eqn:etea_smooth}
	\phi_{\eps} (u) := \phi( \smooth( u ; \eps) ),
\end{align}
where function $\smooth : \real \to \nreal$ is defined as
%------------------------------------------------------------------------------%
\begin{align}\label{eqn:etea_definition_smooth_function}
	\smooth(u ; \eps) := \sqrt{u^{2}+ \eps}, \quad \eps > 0.
\end{align}
%------------------------------------------------------------------------------%
Because $\sqrt{u^2+\eps}$ is always greater than zero, $\phi_{\eps}$ has a continuous first-order derivative
%------------------------------------------------------------------------------%
\begin{align}\label{eqn:etea_dphie}
	\phie' (u) = \frac{u }{\sqrt{u^2+\eps}} \phi'( \sqrt{u^2+\eps} ).
\end{align}
%------------------------------------------------------------------------------%

Table~\ref{tab:etea_penalty} gives the smoothed penalty functions $\phie$ corresponding to $\phi$ in the first column,
and \figref{fig:etea_Example_0_penalty}(ab) illustrates the penalty function and its smoothed version.
The parameter $\eps$ controls the similarity of $\phie$ to non-smooth function $\phi$.
When $\eps \to 0$, $\phie(u) \approx \phi(u)$.
In practice, we set $\eps$ to a very small number (e.g. $10^{-10}$).

%------------------------------------------------------------------------------%
\subsection{Majorization of smoothed penalty function}
%------------------------------------------------------------------------------%

We assume the non-smooth function $\phi$ satisfies the following conditions:
\begin{enumerate}
\item 
	$\phi(u)$ is continuous on $\real$. 
\item 
	$\phi(u)$ is twice continuously differentiable on $\mathbb{R} \backslash \{0\}$.
\item
	$\phi(u)$ is even symmetric:  $\phi(u) = \phi(-u)$.
\item
	$\phi(u)$ is increasing and concave on $\real_{+}$.
\end{enumerate}
Under such assumptions, we find a quadratic function $g: \real \times \real \to \real$,
%------------------------------------------------------------------------------%
\begin{align} 	\label{eqn:etea_majorizer}
	g(u , v ;\phi)  =  \frac{\phi'(v)}{2v} u^2 + \phi(v) - \frac{v}{2}\phi'(v),   
\end{align}
majorizing $\phi$ when $v \neq 0$ \cite[Lemma~1]{Chen_Selesnick_2014_GSSD}.
However,  $\phi$ might not be differentiable at zero, e.g., all functions $\phi$ in Table~\ref{tab:etea_penalty}. 
To avoid this issue, we use \eqnref{eqn:etea_smooth} to ensure the objective function is continuous and differentiable. 
Then the MM method can be applied without the occurrence of numerical issues (e.g., divide by zero).
The following proposition indicates a suitable majorizer of the smoothed penalty functions.
 
%------------------------------------------------------------------------------%
\begin{proposition}\label{pro:etea_pro_1}
If function $\phi$ satisfies the listed four conditions in the previous paragraph and $g$ in \eqnref{eqn:etea_majorizer} is a majorizer of $\phi$, then
\begin{align}\label{eqn:etea_pro1}
	g_{\eps}(u , v ;\phie) = g ( \smooth(u ; \eps) , \smooth(v ; \eps) ;\phi),
\end{align}
is majorizer of $\phie(u) = \phi( \smooth(u ; \eps) ) $ for $u,v \in \real$,
and if writing explicitly, then $g_{\eps}$ in \eqnref{eqn:etea_pro1} is
\begin{align}\label{eqn:etea_smoothed_majorizer}
	g_{\eps}(u , v ;\phie) = \frac{u^2}{2\psi(v)}  + \phie(v) - \frac{v}{2}\phie'(v),
\end{align}
where $\psi(v) = v/\phie'(v)  \neq 0 $.
\end{proposition}
%------------------------------------------------------------------------------%%
%\medskip

The proof of Proposition~\ref{pro:etea_pro_1} is given in \ref{app:etea_A}.
\figref{fig:etea_Example_0_penalty}c and \figref{fig:etea_Example_0_penalty}d give examples of using function \eqnref{eqn:etea_smoothed_majorizer}
to majorize the smoothed penalty functions in \figref{fig:etea_Example_0_penalty}a and \figref{fig:etea_Example_0_penalty}b, respectively.

%------------------------------------------------------------------------------%
\section{Exponential transient excision algorithm}

To formulate the problem of exponential transient excision, we define $\RR$ as
%------------------------------------------------------------------------------%
\begin{equation}
\label{eqn:etea_R}
	\RR := 
	\begin{bmatrix}
		-r		&1 		&		& 				\\
	   	& -r		&1		& 		&  			\\
		&	& 		\ddots &\ddots	&		 	\\
		&	&	&				-r		&1
	\end{bmatrix}.
\end{equation}
%------------------------------------------------------------------------------%
The first-order derivative operator $\DD$ is a special case of $\RR$ with $r=1$. 
In this paper we restrict $ 0 < r < 1 $ to distinguish them.
As a filter,  $\RR$ can be seen to be a first-order filter attenuating low frequencies, with a transfer function
%------------------------------------------------------------------------------%
\begin{align} \label{eqn:etea_Rz}
	R(z) : = 1 - r  z^{-1}.
\end{align}
%------------------------------------------------------------------------------%
Driving the system $ R(z) $ with the step exponential 
\begin{equation}\label{eqn:etea_x}
	x(n) = 
	\begin{cases}
	0, 			& n <  n_0, \\
	r^{n-n_0}, 	& n \ge n_0,
	\end{cases}
\end{equation}
produces an impulse $\delta(n-n_0)$ as an output.
If input signal $\x$ is a step exponential with rate $r$, i.e., a Type 1 artifact, then 
%------------------------------------------------------------------------------%
\begin{equation}	\label{eqn:etea_v}
	\v = \RR\x
\end{equation}
%%------------------------------------------------------------------------------%
will be sparse.

Note that $r$ should be known beforehand.
In practice, we can estimate it from the time-constant of a Type 1 artifact.
Using observation data $y$, if we can measure the time $N_0$ over which a Type 1 transient decays to half its initial height, 
then $r$ can be found by solving $r^{N_0} = 0.5$.
We ignore the influence of the slowly varying baseline, as we assume the exponential decays much faster. 
If the transients have an approximately equal decay rate, then we can use an average measured form multiple transients.

%------------------------------------------------------------------------------%
\subsection{Problem definition}
%------------------------------------------------------------------------------%

In signal model \eqnref{eqn:etea_model}, if an estimate $\hat{\x}$ is known, then we can estimate $f$ by
%------------------------------------------------------------------------------%
\begin{align}	\label{eqn:etea_f}
	\hat \f = (\y - \hat \x) - \HH (\y - \hat \x),
\end{align}
%------------------------------------------------------------------------------%
where $\HH = \BB \AA^{-1}$ is a highpass filter,
where the $\AA$ and $\BB$ matrices are both banded with a structure as \eqnref{eqn:AB}.
Correspondingly, noise $w$ is  $\y - (\hat\x+ \hat\f) = \HH(\y - \hat\x)$,
which indicates a suitable data fidelity term $\norm{ \HH ( \y - \x ) }_2^2$.
Moreover, when the component $x$ is modeled as \eqnref{eqn:etea_x}, we can use $\v  = \RR \x$ as its sparse representation.
Hence, we propose to formulate the estimation of $x$ from noisy data $y$ as the unconstrained optimization problem:
\begin{align}\label{eqn:etea_cost_1}	
	\x{\opt} = \arg \min_{\x}
	\Big\{ 
		P_1(\x) =  \norm{ \HH ( \y - \x ) }_2^2 + \lam \sum_{n}\phie( [\RR \x]_n ) 
	\Big\},				
\end{align}	
where $\lam > 0$ is the regularization parameter related to noise level. 	

%------------------------------------------------------------------------------%
\subsection{Algorithm derivation}
%------------------------------------------------------------------------------%

Using the majorizer of smoothed penalty function \eqnref{eqn:etea_smoothed_majorizer}, 
and recalling that $\HH = \BB \AA^{-1}$ illustrated in \secref{sec:etea_lti_filter},
we majorize the objective function \eqref{eqn:etea_cost_1},
%------------------------------------------------------------------------------%
\begin{align} \label{eqn:etea_cost_mm}
	G(\x , \z ; P_1) 
	= & \ \norm{ \HH ( \y - \x ) }_2^2 + \lam \sum_n g_{\eps} ([\RR \x]_n , [\RR \z]_n ; \phie)		\nonumber\\[0.4em]
	= & \ \norm{ \BB \AA^{-1} ( \y - \x ) }_2^2 + \x^{\tp} \RR^{\tp} \mtxLam{\RR\z} \RR \x + C	
\end{align}
%------------------------------------------------------------------------------%
where $C$ does not depend on $\x$, 
and  $\mtxLam{\RR\z} \in \real^{(N-1) \times (N-1)}$ is a diagonal matrix
\begin{align} \label{eqn:etea_mtxLam}
	\mtxLam{\RR\z}_{n,n}  = \frac{\lam}{2 \psi( [\RR\z]_n ) }
\end{align} 
where $\psi(u) := u / \phie'(u)$, 
which is listed in Table~\ref{tab:etea_penalty}.
%------------------------------------------------------------------------------%
Then the minimizer of \eqref{eqn:etea_cost_mm} is given by
\begin{align}\label{eqn:etea_mm_x}
	\x  =  \Big[ \AA^{-\tp} \BB^{\tp} \BB \AA^{-1} + \RR^{\tp} \mtxLam{\RR\z} \RR \Big]^{-1} \AA^{-\tp}\BB^{\tp} \BB \AA^{-1} \y.			
\end{align}
%------------------------------------------------------------------------------%

Note that, calculating \eqref{eqn:etea_mm_x} directly requires  the solution to a large dense system of equations.
However, since $\AA$ is invertable, we can factor $\AA$ out and write \eqnref{eqn:etea_mm_x} as
%------------------------------------------------------------------------------%
\begin{align} \label{eqn:etea_mm_x_2}
	\x^{(k+1)} = \AA \Big[ \BB^{\tp} \BB + \AA^{\tp} \RR^\tp \mtxLam{\RR\x^{(k)}} \RR\AA \Big]^{-1} \BB^{\tp} \BB \AA^{-1} \y,
\end{align}
%------------------------------------------------------------------------------%
where the MM procedure \eqnref{eqn:etea_mm_iteration} is adopted with $\z = \x^{(k)}$.
Note that the matrix to be inverted in \eqnref{eqn:etea_mm_x_2} is banded,
so that the solution can be computed efficiently by fast banded system solvers (e.g., \cite{ban_Kilic_2008, ban_Kilic_2013}).
Moreover, in \eqnref{eqn:etea_mm_x_2}, all matrix multiplications are between banded ones.
Table~\ref{alg:etea} summarizes the algorithm to solve \eqref{eqn:etea_cost_1}.

\begin{table}[t!]
\caption{Exponential Transient Excision Algorithm.}
\label{alg:etea}
\begin{subequations}\label{eqn:etea}
	\begin{empheq}[box=\fbox]{align*}
		& \text{Input:} ~\y, ~\lam,~r,~\AA,~\BB  		\nonumber		\\
		& \text{Initialization:} ~\x \in \mathbb{R}^{N} 				\\
		& \b \mygets  ~\BB^{\tp}\BB \AA^{-1} \y             			\\
		& \text{Repeat } \nonumber \\
		& \qquad 
			\v \mygets \ \RR \x  		\\
		& \qquad 
			[\LAM]_{n,n} \mygets  \frac{\lam}{2 \psi(v(n))} 			\\
		& \qquad
			\QQ	 \mygets \BB^{\tp} \BB +\AA^{\tp}\RR^{\tp} \LAM \RR\AA \\
		& \qquad
			\x\ \mygets \AA \QQ^{-1} \b	\\
		& \text{Until convergence } 					\nonumber \\
		& \f \mygets ( \y - \x ) - \HH ( \y - \x ) \\
		& \text{Return: } \x,~\f. \nonumber
	\end{empheq}
\end{subequations}
\end{table}

\subsection{Optimality condition and parameter selection}

For problem \eqnref{eqn:etea_cost_1}, it is difficult to derive optimality conditions directly.
Here, we consider the optimality condition indirectly via an equivalent problem,
similar to the discussion of optimality condition for total variation problems in \cite{Bach_2012_now, Fuchs_2004_Tinfo}
and for $\LPFTVD$ in \cite{Selesnick_2013_lpftvd}.

To facilitate our derivation, we firstly denote the optimal solution of \eqnref{eqn:etea_cost_1} by $\x^{*}$.
We then define $\v^{*} = \RR \x^{*}$ where $\RR$ is given by \eqnref{eqn:etea_R}.
We define $\GG \in \real^{N \times (N-1)}$,
\begin{equation}
	\GG := 
	\left[
	\begin{array}{llccccc}
		0\\
		1			&0 												\\
		r			&1 		&0		& 								\\
	   	r^2			&r 		&1		&0 		&  						\\
		\vdots		&		& 		&\ddots 	&\ddots	&	\\
		r^{N-3}	&\cdots		&\cdots		&r		& 1  	&0			\\
		r^{N-2}	&\cdots		&\cdots		&r^2 	&r		& 1  
	\end{array}	
	\right]			
	\label{eqn:S}	
\end{equation}
which satisfies
\begin{align}\label{eqn:RS}	
	\RR\GG = \II.
\end{align}
The operator $\GG$ acts as an inverse filter of $\RR$.
Using the above assumptions, we derive the following proposition regarding optimality conditions for \eqnref{eqn:etea_cost_1}.

%------------------------------------------------------------------------------%
\begin{proposition}\label{pro:etea_pro_2}
If $\x^{*}$ is an optimal solution to \eqnref{eqn:etea_cost_1}, then $\v^{*} = \RR \x^{*}$ is an optimal solution to problem
\begin{align}\label{eqn:etea_problem_v}
	\v^{*} = \arg \min_{\v}   Q(\v) ,	
\end{align}
where
\begin{align}
	Q(\v) = \norm{ \HH ( \y_0 - \GG\v ) }_2^2  + \lam \sum_n \phie ([\v]_n).
\end{align}
and $\y_0 = \y - (\x^{*} - \GG \RR\x^{*})$.
\end{proposition}

\begin{proof}
We define $\x_0 \in \real^{N}$ as
\begin{align}\label{eqn:etea:x0}
	\x_0 := \GG \v^{*} = \GG \RR \x^{*}
\end{align}
where $\x^{*}$ is the optimal solution to \eqnref{eqn:etea_cost_1}.
Note that $\x_0$ and $\x^{*}$ share an identical  $\v^{*} = \RR \x_0 = \RR \x^{*}$.
Moreover, we define the difference between $\x^{*}$ and $\x_0$ as
\begin{align}\label{eqn:etea:d0}
	\d := \x^{*} - \x_0.
\end{align}
It follows that $\RR \d = \RR \x^{*} - \RR \x_0 = \mathbf{0}$. As a consequence, $\x_0$ is the optimal solution to problem
\begin{subequations}\label{eqn:etea_pro_1}
\begin{align}
	\x_0 = 	&	\arg \min_{\x} P_1(\u)  \\
			& 	\ \text{s.t. } \x = \u - \d.
\end{align}
\end{subequations}
Note $ x_0(0) = 0 $ must be satisfied in \eqnref{eqn:etea_pro_1} because of the definition of $\GG$.
Thus, we can write an equivalent problem to \eqnref{eqn:etea_pro_1} using \eqnref{eqn:etea:x0},
\begin{subequations}\label{eqn:etea_pro_2}
\begin{align}
	\{\x_0, \v^{*} \} = 	\arg \min_{\x,\v} P_1( \u) 	\\
							\text{s.t. }  \x = \u - \d	\\
							 \x = \GG \v \quad
\end{align}
\end{subequations}
where $\{\x_0, \v^{*} \}$ must be the optimal solution.
In problem \eqnref{eqn:etea_pro_2}, $\u$ is uniquely determined by linear functions of $\v$ and $\d$, 
so problem \eqnref{eqn:etea_pro_2} can be simplified by substituting the variables,
\begin{subequations}\label{eqn:etea_pro_3}
\begin{align}
	\{\x_0, \v^{*} \}= 	\arg \min_{\x,\v} Q( \v) 	\\
		\text{s.t. }  \x = \GG \v
\end{align} 
\end{subequations}
where $Q(\v) = P_1(\u)$, with $\u = \GG \v +\d$, and can be written explicitly,
\begin{align}\label{eqn:etea_Qv1}
	Q(\v) & = \norm{ \HH ( \y - \GG\v -\d) }_2^2  + \lam \sum_n \phie ([\RR\GG\v + \RR\d]_n).
\end{align}
Because $\RR\GG = \II$ and $\RR\d = \mathbf{0}$, we simplify \eqnref{eqn:etea_Qv1} as
\begin{align}\label{eqn:etea_Qv2}
		Q(\v) &  = \norm{ \HH ( \y_0 - \GG\v ) }_2^2  + \lam \sum_n \phie ([\v]_n)
\end{align}
where $\y_0 = \y - \d$. 

In this case, the equality constraint in \eqnref{eqn:etea_pro_3} is redundant.
We can solve the unconstrained problem \eqnref{eqn:etea_problem_v} first, and then compute $\x_0$ by \eqnref{eqn:etea:x0}.
This implies that $\v^{*}$ is the optimal solution to  \eqnref{eqn:etea_problem_v} and satisfies $\nabla Q (\v^{*}) = \mathbf{0}$.
\end{proof}

Using Proposition~\ref{pro:etea_pro_2}, instead of problem \eqnref{eqn:etea_cost_1},
we alternatively consider the optimality condition of problem \eqnref{eqn:etea_problem_v}
where as long as $\x^{*}$ is an optimal solution of \eqnref{eqn:etea_cost_1}, $\v^{*}$ is an optimal solution of \eqnref{eqn:etea_problem_v}.
Moreover, we can rewrite the optimality condition of  \eqnref{eqn:etea_problem_v} as
%------------------------------------------------------------------------------%
\begin{align}
	p(n) = \lam \phie'([\v^{*}]_n) = \lam \phie'( [\RR \x^{*}]_n ),
	\label{eqn:etea_opt_condition}
\end{align}
%------------------------------------------------------------------------------%
where $\p = 2 \GG^{\tp} \HH^{\tp}\HH(\y_0 - \x_0)$, which can be rewritten as
%------------------------------------------------------------------------------%
\begin{align}
	\p = 2 \GG^{\tp} \HH^{\tp}\HH(\y- \x^{*}).
	\label{eqn:etea_p}
\end{align}
%------------------------------------------------------------------------------%
Note that, writing $\p$ as \eqnref{eqn:etea_p} does not require the solution to problem \eqnref{eqn:etea_problem_v} or to compute $\d$ and $\y_0$.
Therefore, although we derived an indirect way to verify the optimality condition for \eqnref{eqn:etea_cost_1}, the final procedure is direct.

\medskip
\textbf{Setting parameter $\lam$.}
The equation \eqnref{eqn:etea_opt_condition} can be used as a guide to set the regularization parameter $\lam$.
Suppose the observation data is composed of Gaussian noise only, then a proper value of $\lam$ should make the solution of \eqnref{eqn:etea_cost_1} almost identically zero.
Thus, its sparse representation $\v^{*}$, given by \eqnref{eqn:etea_v}, should be all zero as well.
In this case, we can calculate a vector $\q$, which depends only on noise,
%------------------------------------------------------------------------------%
\begin{align}
	\q = 2 \GG^{\tp} \HH^{\tp}\HH \w.
\end{align}
%------------------------------------------------------------------------------%
When $\p$ is calculated by \eqnref{eqn:etea_p}, we find a constraint that $-\lam < p(n) < \lam$, by the property of $\phie$.
Furthermore, since $\x^{*}\approx \mathbf{0}$ is expected when the observation is pure noise (i.e., $\y = \w$), 
the values of $\q \approx \p$ can be considered bounded 
%------------------------------------------------------------------------------%
\begin{align}
	q(n) \in [-\lam, +\lam],  
	\quad \text{when } \abs{v(n)} \le \beta_{\eps}, \text{ for } n \in \mathbb{Z}_N
	\label{eqn}
\end{align}
%------------------------------------------------------------------------------%
with high probability. 
The value $\beta_{\eps}$ here is related to $\eps$, and since $\eps$ is extremely small, it can be simply assumed to be 0.
As a consequence, the value of $\lam$ needs to satisfy
%------------------------------------------------------------------------------%
\begin{align}
	\lam \ge \max\big\{ \abs{2 [ \GG^{\tp} \HH^{\tp}\HH \w ]_n}, n \in \mathbb{Z}_N \big\}.
\end{align}
%------------------------------------------------------------------------------%
Further, if the statistical property of the noise can be exploited, e.g., $w(n) \sim \mathcal{N}(0, \sigma_w^2)$,
$\lam$ can be set statistically, such that
\begin{align}\label{eqn:etea_three_sigma_h_1}
	\lam = 2.5 \sigma_{w} \norm{ 2 \h_1 }_2
\end{align}
where $\sigma_{w}$ is the standard deviation of the noise and $\h_1$ is the impulse response corresponding to system $\GG^{\tp} \HH^{\tp}\HH$.
Note that, $\GG^{\tp} \HH^{\tp}\HH$ is a filter with a transfer function  $P_1(z) = H^2(z) / {R(z)}$.
Hence, although we give the matrix $\GG$ to derive the approach to set $\lam$, 
it is not required to compute $\h_1$
because the filter $P_1(z)$ can be implemented directly using $H(z)$ in \eqnref{eqn:etea_banded_filter}, and $R(z)$ in \eqnref{eqn:etea_Rz}.

%------------------------------------------------------------------------------%
\subsection{Example: synthetic signal}\label{sec:etea_example_1}
%------------------------------------------------------------------------------%

%------------------------------------------------------------------------------%
\begin{figure}[!t]
\centering	
	\includegraphics[scale = \figurescale]{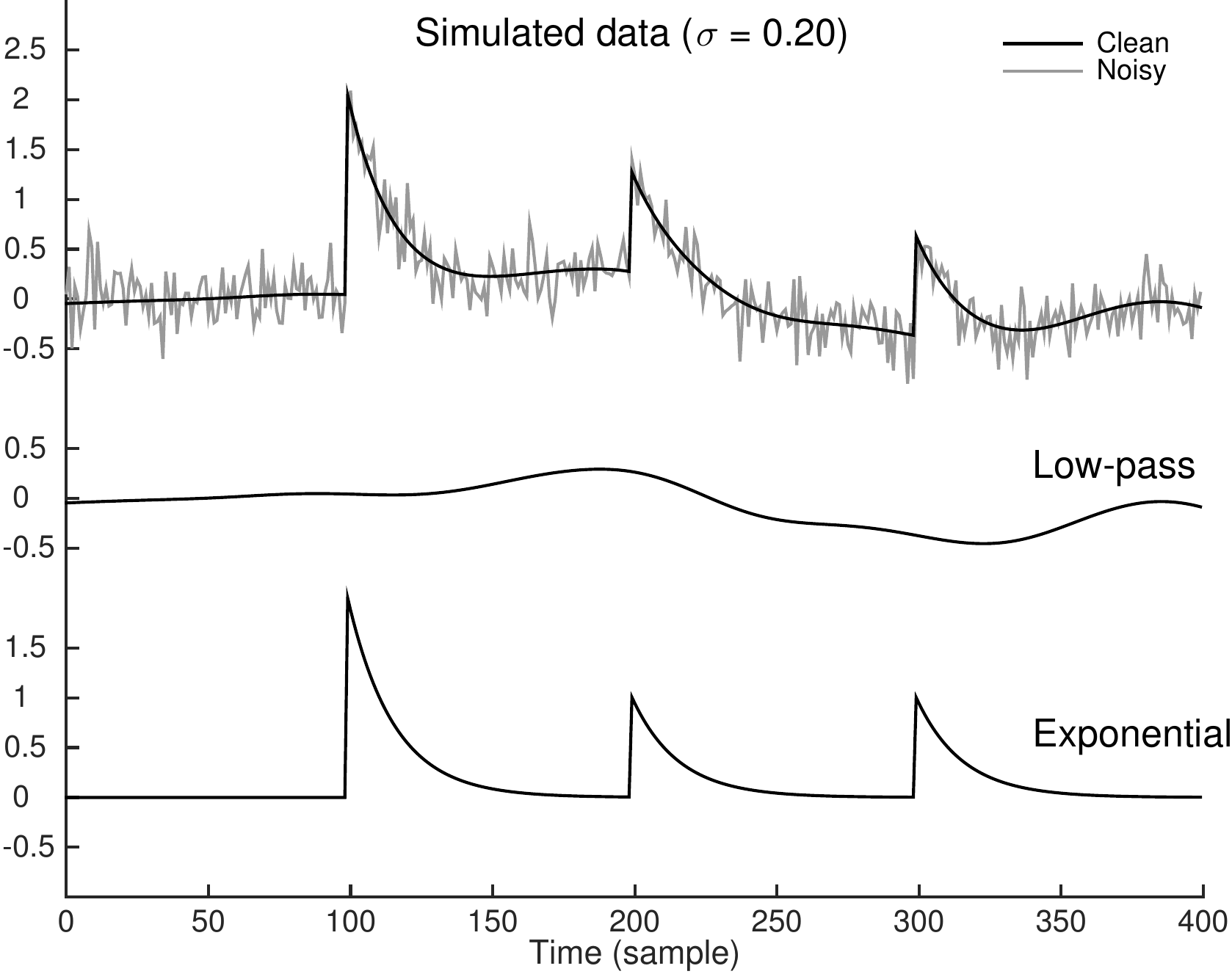} 
\caption{\textbf{Example 1:}~Simulated data comprising a lowpass component and exponential transients.}
\label{fig:etea_Example_1_data}
\end{figure}
%------------------------------------------------------------------------------%
%------------------------------------------------------------------------------%
\begin{figure}[!t]
\centering	
	\includegraphics[scale = \figurescale]{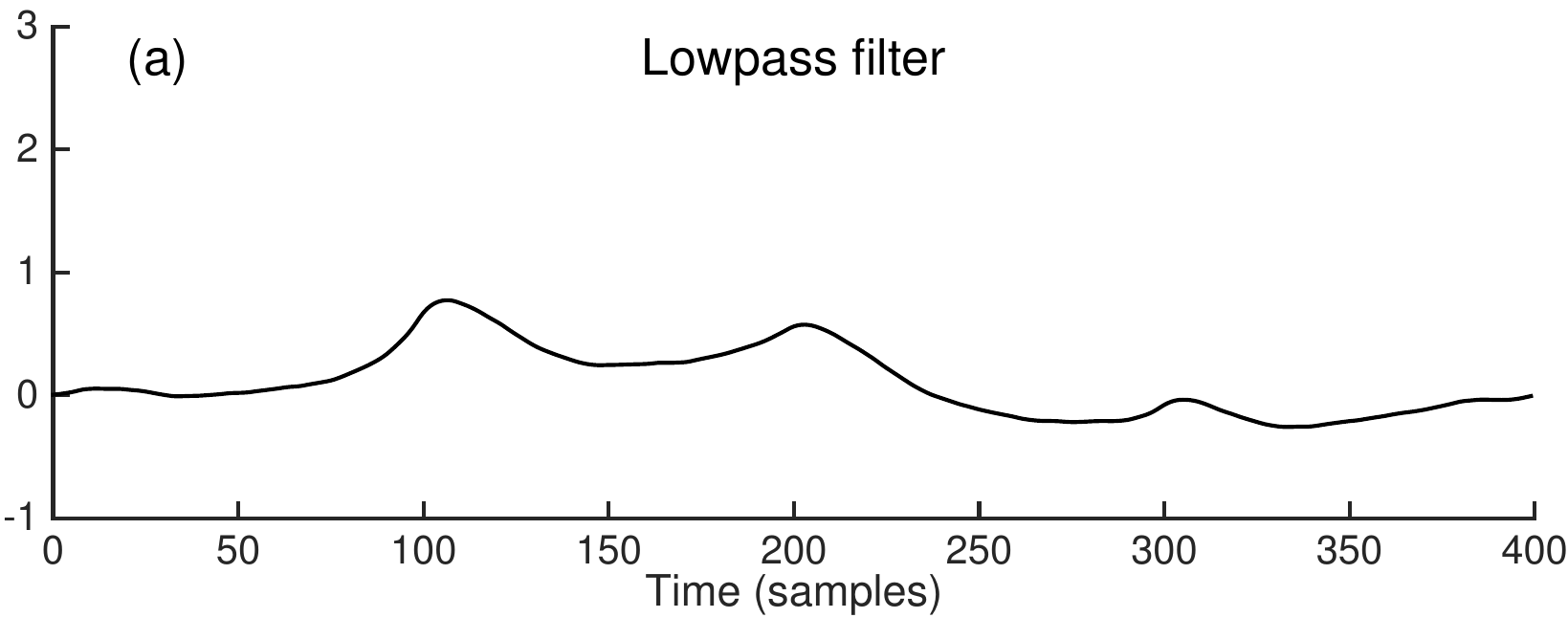}   	\\
	\includegraphics[scale = \figurescale]{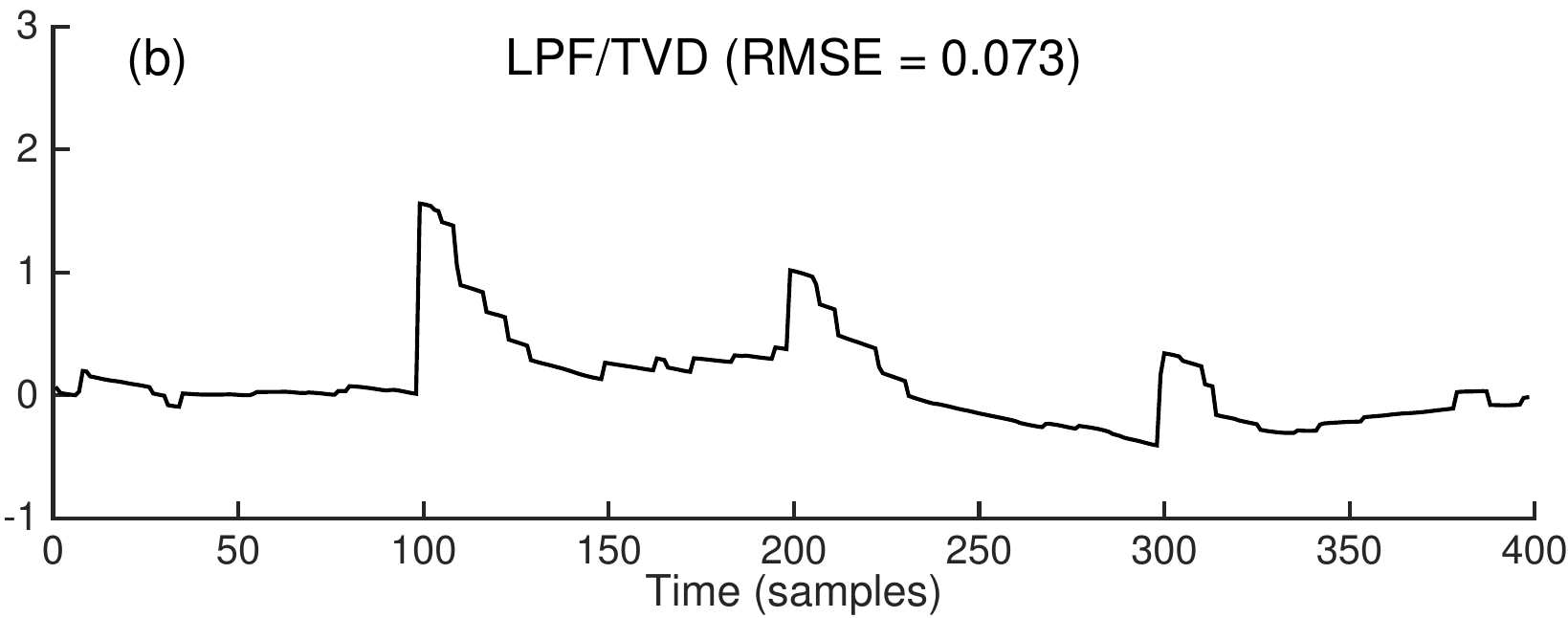} 	\\
	\includegraphics[scale = \figurescale]{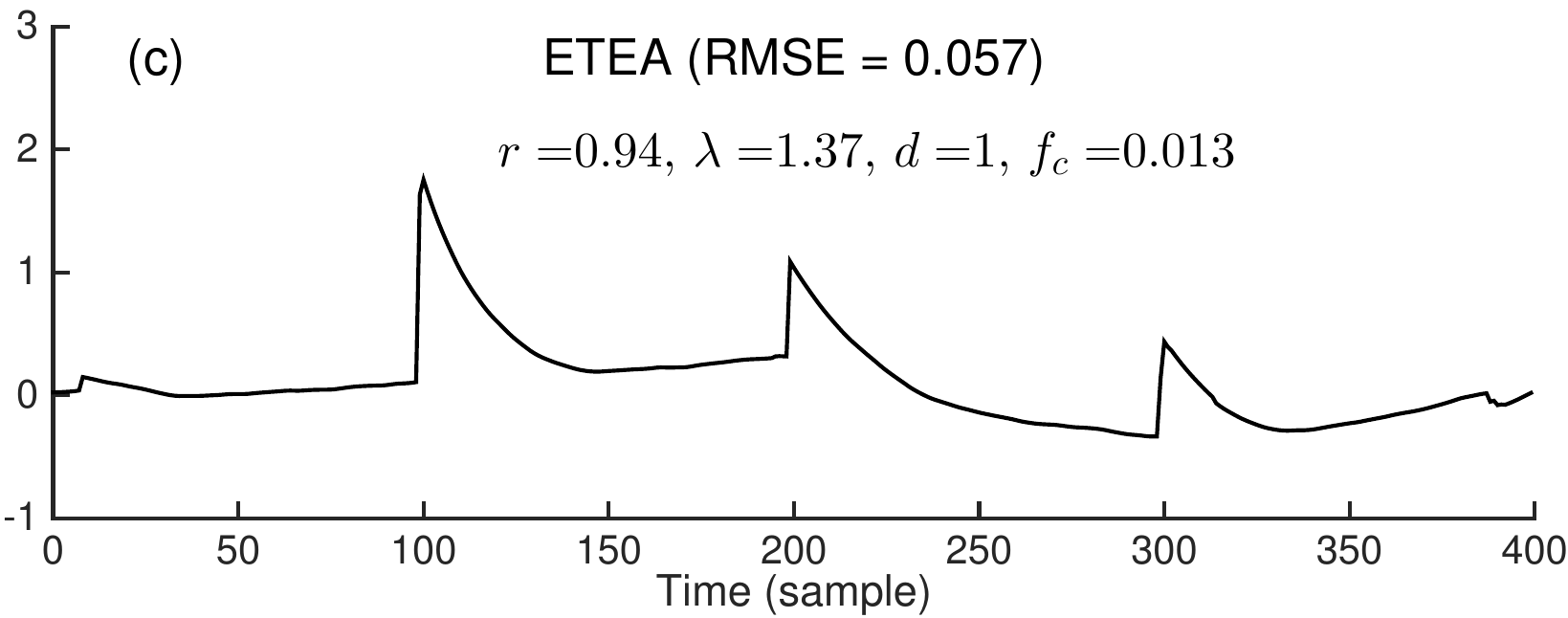} 
	\caption{\textbf{Example 1:}~
		The denoising results of (a) lowpass filtering, (b) $\LPFTVD$, (c) $\name$ (proposed method).}
	\label{fig:etea_Example_1_results}
\end{figure}
%------------------------------------------------------------------------------%\
%------------------------------------------------------------------------------%
\begin{figure}[!t]
\centering
	\includegraphics[scale = \figurescale]{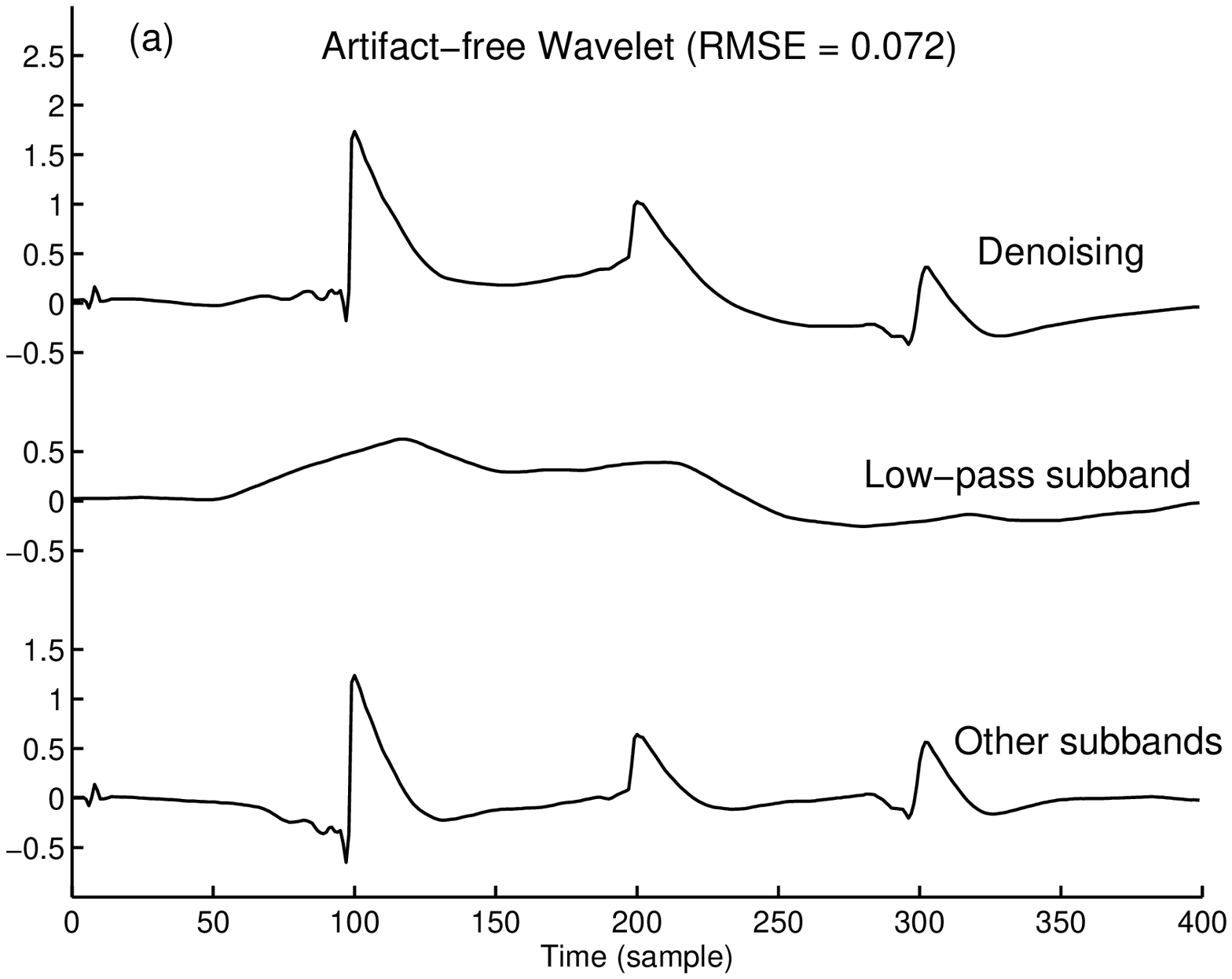}	\\ 	
	\includegraphics[scale = \figurescale]{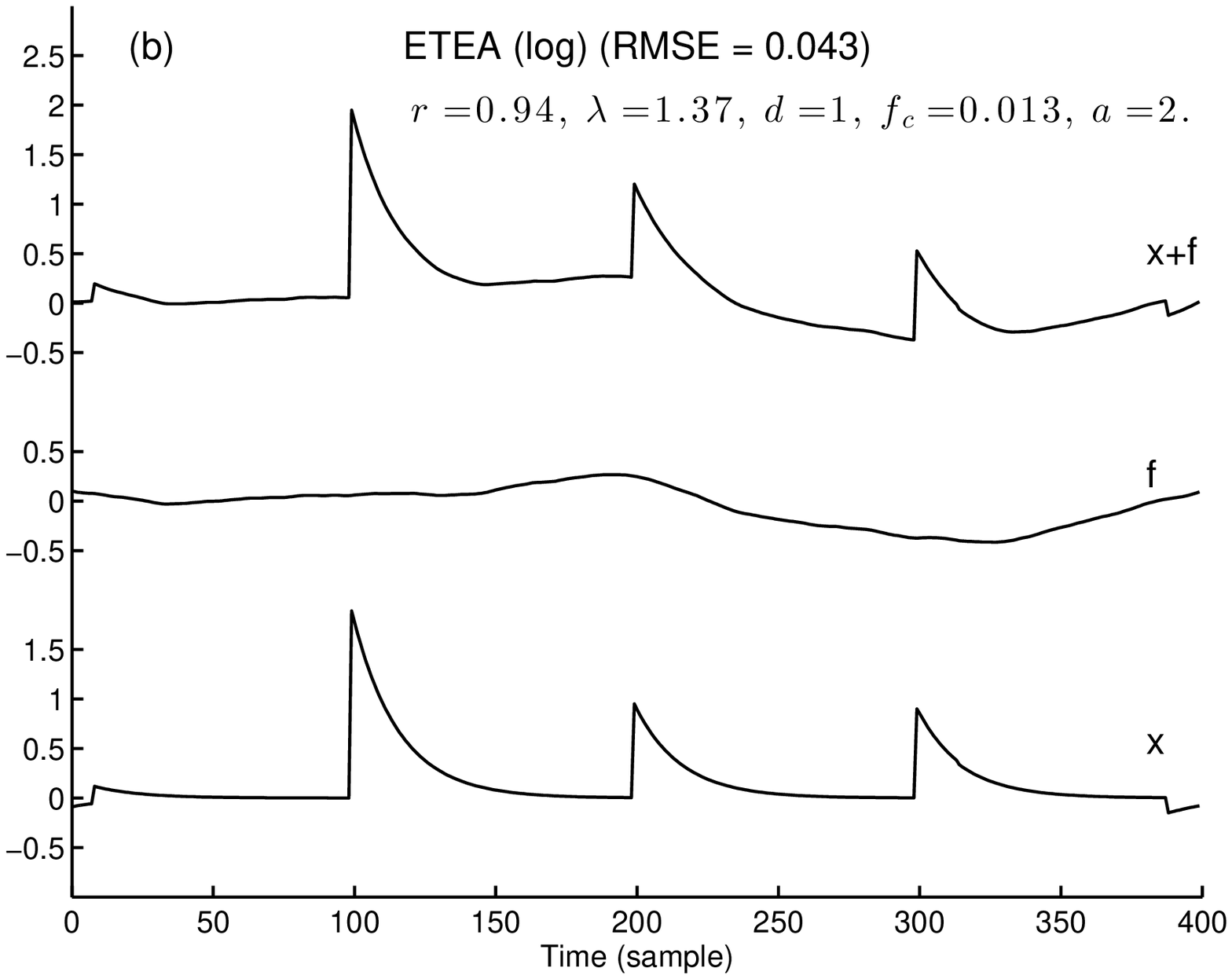} 
	\caption{\textbf{Example 1:}~Comparison of denoising and decomposition results using 
		(a) Wavelet-based method and, 
		(b) $\name$ with non-convex regularization.}
\label{fig:etea_Example_1_decomp}
\end{figure}
%------------------------------------------------------------------------------%

To illustrate $\name$, the algorithm is tested on a simulated signal and compared  to other methods. 
\figref{fig:etea_Example_1_data} shows the test signal and its components:
a lowpass signal $f$, a transient component $x$, composed of three exponential decays,
and white Gaussian noise with $\sigma_w = 0.20$. 
The noisy observation is shown in \figref{fig:etea_Example_1_data} in gray.
\medskip

\textbf{Denoising.}
The result of conventional lowpass filtering is shown in \figref{fig:etea_Example_1_results}a.
The result is smooth, but the step edges of each pulse are diminished.
Moreover, if compared to the true low-pass baseline, the result is also distorted.
The result of using $\LPFTVD$ \cite{Selesnick_2013_lpftvd} is illustrated in \figref{fig:etea_Example_1_results}b.
Although the abrupt jumps have been preserved, serious staircase effects appear.
$\LPFTVD$ uses the first-order derivative operator $\DD$ in regularization.
The derivative of an exponential is another exponential,
therefore the derivative is unlikely to have a sufficiently sparse representation. 
A compensative method is to widen the bandwidth of the lowpass filter, however, this leads to a noisy result. 
The result presented in \figref{fig:etea_Example_1_results}b is selected by tuning all necessary parameters of $\LPFTVD$
to optimize RMSE (root-mean-square error) value.

The result of the proposed ETEA method is shown in \figref{fig:etea_Example_1_results}c. 
The step edges and decay behavior are well preserved (RMSE = 0.057). 
We use a smoothed $\ell_1$-norm penalty function.
In this example, we set $\eps = 10^{-10}$, $r = 0.94$, filter order parameter $d = 1$, 
cut-off frequency $f_c = 0.013 \mathrm{~cycle/sample}$, and regularization parameter is calculated by \eqnref{eqn:etea_three_sigma_h_1}
based on the noise principle.

\medskip
\textbf{Decomposition.}
Wavelet-based methods for artifact correction have been described in \cite{Akhtar_2012_SP, artifacts_Islam_2014, Molavi_2012, Sato_2006_NeuroImage}.
In \figref{fig:etea_Example_1_decomp}a, 
we illustrate denoising and decomposition results obtained using the stationary (un-decimated) wavelet transform \cite{CD95}
with Haar wavelet using hard-threshold determined by the $\sqrt{2\log{N}} \sigma_{w}$ thresholding scheme. 
The denoised result is further enhanced by the artifact-free denoising method \cite{Durand_2001_ICASSP, DurandFroment_2003_SIAM}
which uses total variation minimization to overcome pseudo-Gibbs oscillations.
Although the denoising output is relatively smooth and captures the discontinuities, the decomposed components are both distorted.
Compared to the true components in \figref{fig:etea_Example_1_data},
the lowpass subband deviates from true lowpass component (especially at about $n=100$).
The transient component, reconstructed from other subbands after denoising, has to compensate for the error.
Hence, the estimated transient component does not have a zero baseline (see bottom of \figref{fig:etea_Example_1_decomp}a).

Non-convex penalty functions can be used to improve the result of ETEA.
Here we use the smoothed non-convex logarithm penalty function in Table~\ref{tab:etea_penalty}  ($a = 2$).
The filtering result is shown in \figref{fig:etea_Example_1_decomp}b, where discontinuities are more accurately preserved, and the RMSE is reduced to 0.043
compared with the result in \figref{fig:etea_Example_1_results}c.
We also illustrate the decomposed $f$ and $x$ components in \figref{fig:etea_Example_1_decomp}b.
$\name$ recovers both of the components accurately. 
The algorithm run-time is not affected by the choice of penalty function.
In this example, 50 iterations of the algorithm takes about 25 ms on a MacBook Pro 2012 with a 2.7 GHz CPU,
implemented in Matlab 8.4.

\subsection{Example: artifact removal of ECoG data}

Conventional EEG studies and most clinical EEG applications are restricted below 75 Hz \cite{Rangayyan_book}.
Advanced measuring methods such as ECoG records multichannel cortical potentials from the micro-electrodes located inside human skull with a higher sampling rate.
In this example, we use a 5-second ECoG signal epoch, with a sampling rate of 2713 Hz,
recorded from a mesial temporal lobe epilepsy (MTLE) patient.

\figref{fig:etea_Example_2_x}a shows the raw data in gray and the corrected data by ETEA in black.
There are two Type~1 artifacts identified in this 5-second epoch. 
One is at about $t = 1.20$ second, and the other is at about $t = 3.60$ second.
In signal model \eqnref{eqn:etea_model}, suppressing component x which represents artifacts,
the corrected data $(f+w)$ should preserve the components.
We show the corrected data in \figref{fig:etea_Example_2_x}a, where the two Type~1 spikes are correctly removed.
The decomposed artifact is illustrated in \figref{fig:etea_Example_2_x}a as well, which adheres a zero-baseline.

Since wavelet-based methods have been successfully applied to suppress artifacts 
\cite{Akhtar_2012_SP, artifacts_Islam_2014, Molavi_2012}.
a comparison with a wavelet-based method is shown in \figref{fig:etea_Example_2_x}b.
We use the stationary (un-decimated) wavelet transform \cite{CD95} 
with Haar wavelet filter and the non-negative garrote threshold function \cite{wave_Gao_1998},as recommended in \cite{artifacts_Islam_2014}.
The thresholding has been applied to all the subbands except the lowpass band,
and the artifact component is obtained by subtracting the corrected data from the raw data.
As shown, this approach estimates transient pulses but the estimated pulse adheres to the shape of the Haar wavelet filter: a positive-negative pulse 
(see the bottom square box in \figref{fig:etea_Example_2_x}b), but not the true Type~1 artifact in \figref{fig:etea_Example_0_type}b.

To more clearly illustrate the estimated results by the two methods, 
we show the details of the corrected data and the artifact from $t = 3.50$ to $3.65$ second in dashed line boxes
in \figref{fig:etea_Example_0_type}a and \figref{fig:etea_Example_0_type}b, respectively.
Since the wavelet-based method cannot correctly estimate Type~1 artifact in this case, 
a `bump' can be observed in \figref{fig:etea_Example_2_x}b to compensate the error.
Moreover, this inappropriate estimation will cause the corrected data to change before Type~1 artifact actually occurs
(see \figref{fig:etea_Example_0_type}b).
In contrast, ETEA estimates the artifact as the abrupt drift with a decay, 
then it exicises the artifact without influencing the data before the transient occurs (see \figref{fig:etea_Example_0_type}a).

In this example, for the signal with a length of $13565$ samples,
the proposed algorithm converges within 40 iterations, and takes about 330 ms on a MacBook Pro 2012 with 2.7 GHz CPU,
implemented in Matlab 8.4.
The cost function history is shown in \figref{fig:etea_Example_2_cost}.

%------------------------------------------------------------------------------%
\begin{figure}[!t]
\centering
	\includegraphics[scale = \figurescale] {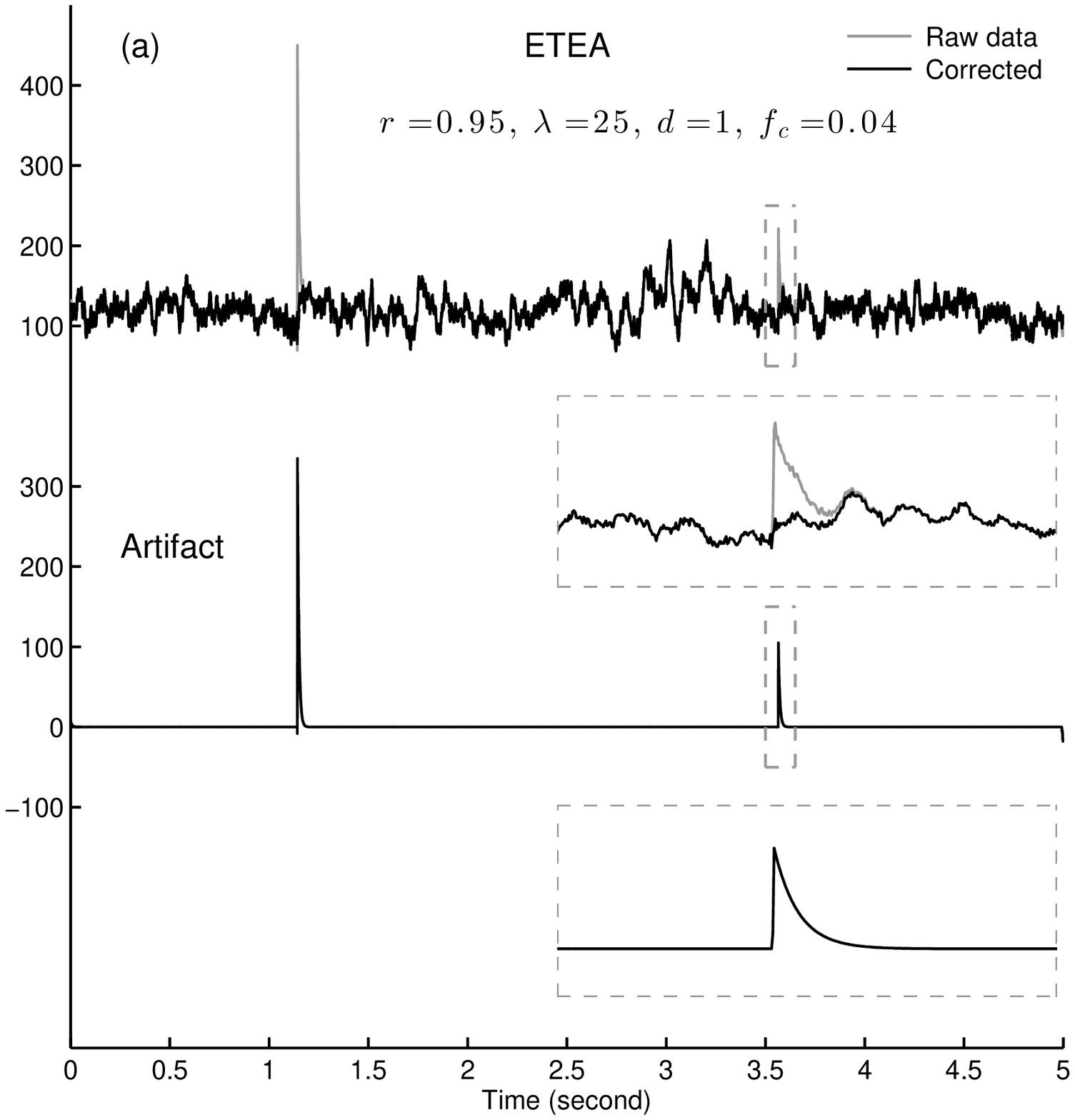} \\
	\includegraphics[scale = \figurescale] {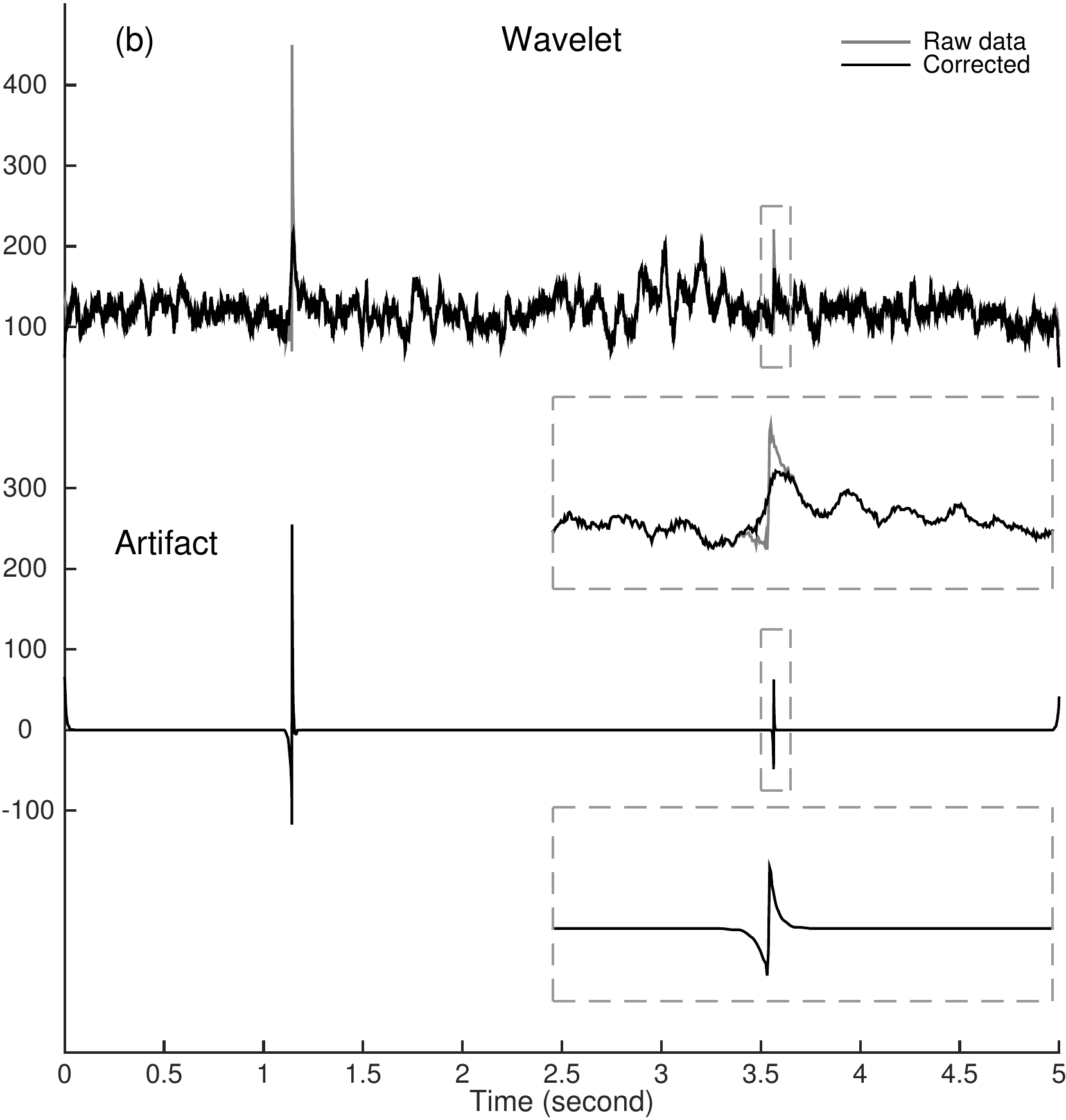}
	\caption{\textbf{Example 2:}~ Comparison of artifacts removal for ECoG data. 
		(a) ETEA (proposed method), (b) wavelet-based method.}
	\label{fig:etea_Example_2_x}
\end{figure}
%------------------------------------------------------------------------------%
%------------------------------------------------------------------------------%
\begin{figure}[!t]
\centering
	\includegraphics[scale = \figurescale]{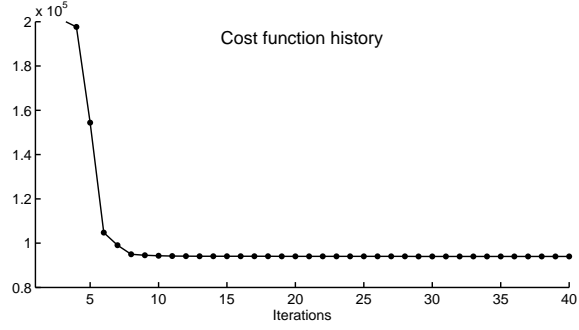} \\
	\caption{\textbf{Example 2:}~Cost function history.}
	\label{fig:etea_Example_2_cost}
\end{figure}
%------------------------------------------------------------------------------%

%------------------------------------------------------------------------------%
\section{Higher-order ETEA}   \label{sec:etea_higher_order_etea}
%------------------------------------------------------------------------------%

%------------------------------------------------------------------------------%
\begin{figure}[!t]
\centering
	\includegraphics[scale = \figurescale] {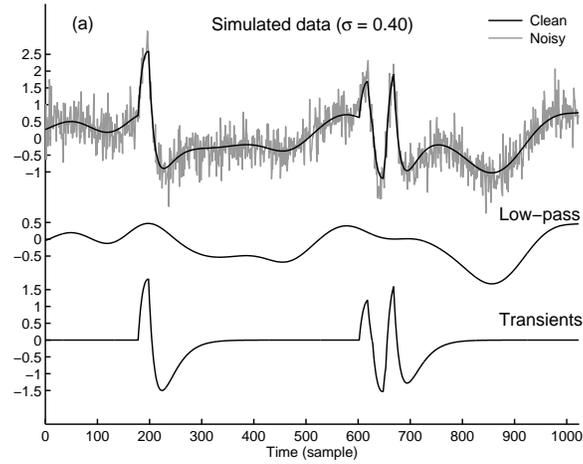} \\
	\caption{\textbf{Example 3:}~The simulated test data, 
		generated by driving an AR filter with a sparse sequence and adding noise.} 
	\label{fig:etea2_Example_1_data}
\end{figure}
%------------------------------------------------------------------------------%
%------------------------------------------------------------------------------%
\begin{figure}[!t]
\centering
	\includegraphics[scale = \figurescale] {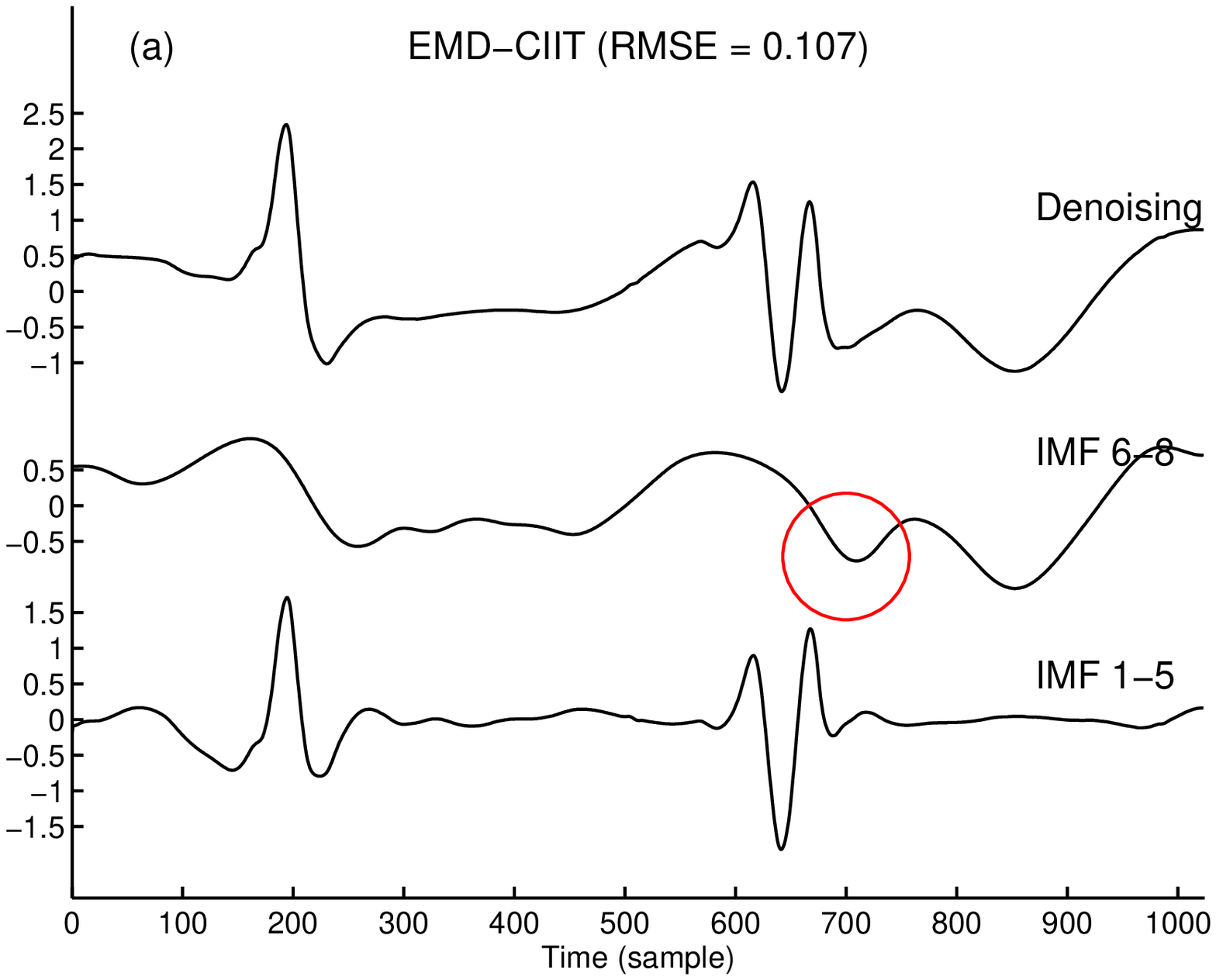} \\
	\includegraphics[scale = \figurescale] {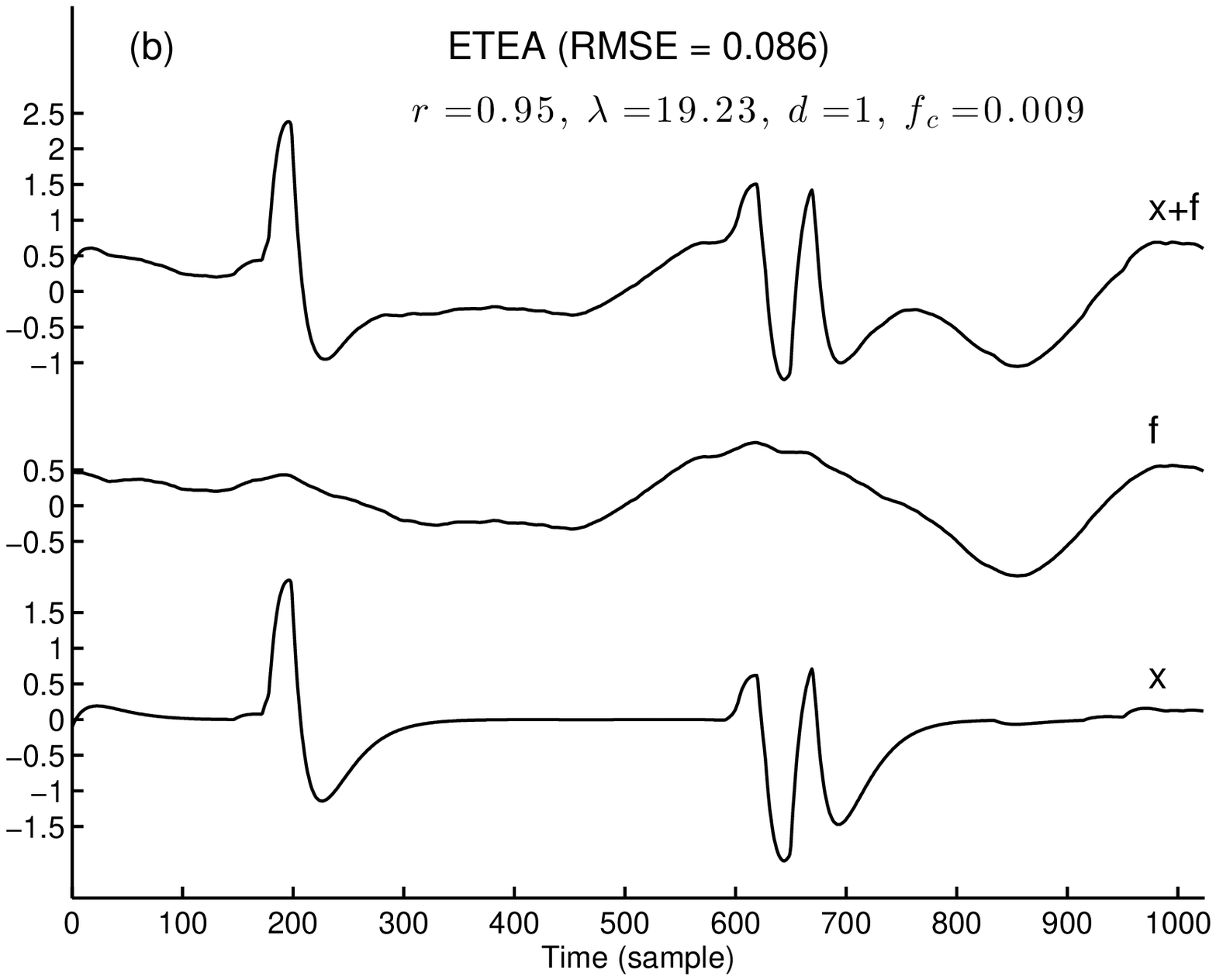} \\
	\includegraphics[scale = \figurescale] {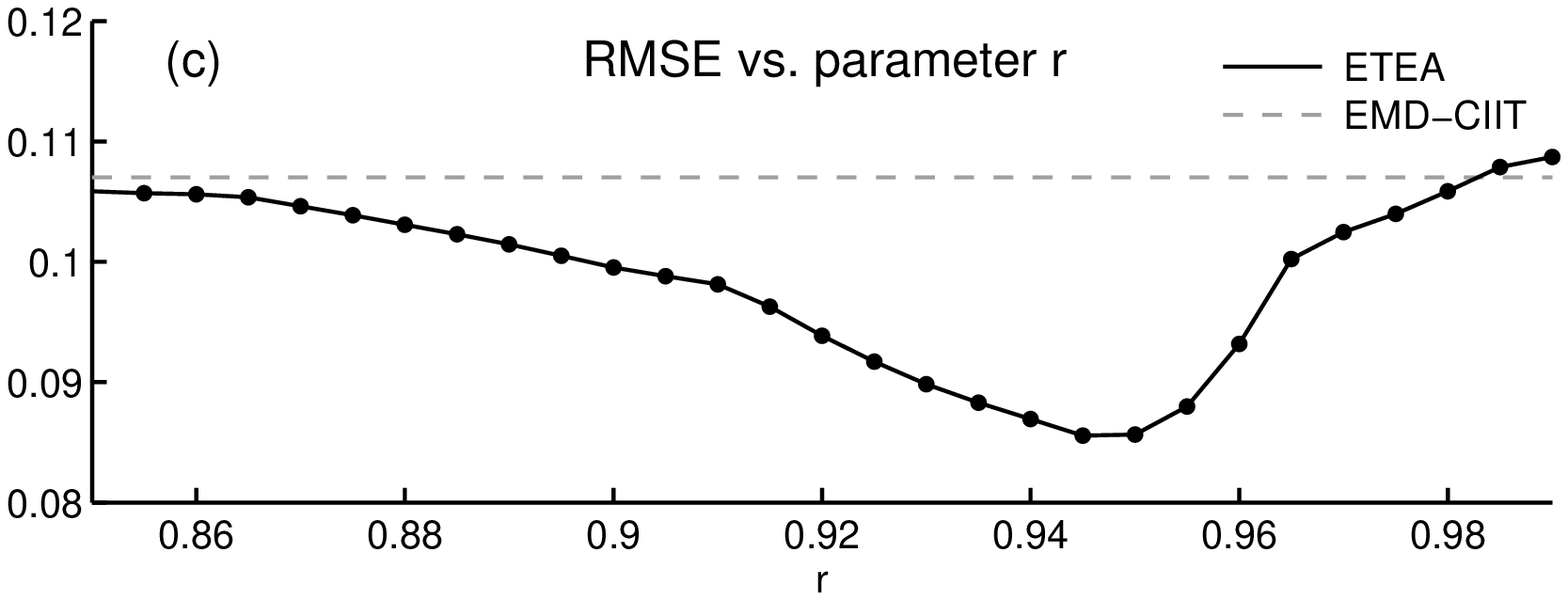} 	
	\caption{\textbf{Example 3:}~Denoising and decomposition results by (a) EMD based denoising (EMD-CIIT \cite{emd_Kopsinis_tsp_2009}) and (b) $\name$.}
	\label{fig:etea2_Example_1_x}
\end{figure}
%------------------------------------------------------------------------------%

In the previous section, we have presented $\name$ based on signal model \eqnref{eqn:etea_model}, suitable for Type~1 artifacts, 
where the component $x$ has discontinuous step exponential transients.
Here we illustrate another version of $\name$ based on signal model \eqnref{eqn:etea_model}, where the component $x$ models Type~0 artifacts.

First, we consider the operator 
%------------------------------------------------------------------------------%
\begin{equation}
\label{eqn:etea_DDB}
	\RR_2 := 
	\begin{bmatrix}
		r^2		&-2r &1 		&		& 			\\
	   	&r^2 & -2r		&1		& 		&  			\\
		&	&\ddots & 		\ddots &\ddots	&		 \\
		&	&		&r^2		&-2r		&1
	\end{bmatrix},
\end{equation}
%------------------------------------------------------------------------------%
where $\RR_2 \in \real ^{(N-2) \times N} $ has three non-zero coefficients in each row. 
As a special case, when $r=1$, $\RR_2$ is the second-order difference operator,
and $\norm{\RR_2 \x}_1$ is the same as the regularizer used in \cite{l1_Kim_2009} for $\ell_1$ detrending.

Taking $\RR_2$ as a filter, the transfer function is
\begin{align}\label{eqn:etea_R2z}
	R_2(z) = ( 1 - r z^{-1})^{2},
\end{align}
which has a double zero at $ z = r $. 
The impulse response of system $R_2^{-1}(z)$ is $ (n+1)r^{n} $, when $n \ge 0$.
It has the same shape as \figref{fig:etea_Example_0_type}a,
which is a suitable model for the Type 0 artifacts in \cite{artifacts_Islam_2014}.
Therefore, $\RR_2 \x$ is sparse when $x$ is composed of such piecewise smooth transients.
Similar to \eqnref{eqn:etea_cost_1}, we formulate the problem
%------------------------------------------------------------------------------% 
\begin{align}\label{eqn:etea_cost_R2}
	\x\opt = \arg\min_{\x} \Big\{ P_2(\x) = \norm{ \HH ( \y - \x ) }_2^2 + \lam \sum_n \phie ( [\RR_2 \x]_n) \Big\}			
\end{align}	
%------------------------------------------------------------------------------%
to estimate the transient component $x$, so that the corrected data is $y-\hat x$,
and the low-pass trend $f$ can be estimated by \eqnref{eqn:etea_f}.
Additionally, \eqnref{eqn:etea_cost_R2} can be easily solved by ETEA (in Table \ref{alg:etea}) substituting $\RR$ by $\RR_2$.

An example of complicated and irregular artifacts is the eye blink/movement artifacts in EEG data, which may have various morphologies and durations.
In this case, we assume that the artifacts can be estimated by a continuous piecewise smooth waveform, 
generated by applying a certain sparse impulse sequence to $R_2^{-1}(z)$ \eqnref{eqn:etea_R2z}.
In other words, we broaden our signal model so that the artifact is not only equivalent to an isolated Type 0 transient as in \cite{artifacts_Islam_2014}, 
but also a superposition of multiple such transients, with some freedom of scaling and shifting.
The problem of estimating $x$ can be solved by \eqnref{eqn:etea_cost_R2} in this case as well. 

\subsection{Example: Simulated data}

\figref{fig:etea2_Example_1_data} shows the simulated data and its lowpass and transient components.
The transient component has several pulses with different heights and widths.
They are obtained by feeding a sequence of impulses into system $R_2^{-1}(z)$.
The filter is given in \eqnref{eqn:etea_ARMA} with $a_k = [1, -2r , r^2 ]$, $b_0 = 1$, and $r = 0.950$.

Empirical mode decomposition (EMD) \cite{Huang_1998_EMD, emd_Huang_2003}, a powerful method for analyzing signals, has been successfully utilized in different fields, 
including neuroscience, biometrics, speech recognition, electrocardiogram (ECG) analysis, and fault detection 
\cite{emd_Chang_2009_iris, emd_Fleureau_2011, emd_Huang_2006_speech, emd_Mandic_2013, Rilling_2008_TSP, emd_Tang_sp_2012, emd_Xie_sp_2014_face, emd_Yan_2014_ECG}. 

As a comparison to the proposed approach, we use the EMD based denoising algorithm in \cite{emd_Kopsinis_tsp_2009},
which uses wavelet coefficient thresholding techniques on decomposed IMFs.
More specifically, we use clear first iterative interval thresholding (CIIT) with smoothly clipped absolute deviation (SCAD) penalty thresholding                     
\cite{emd_Kopsinis_EUSIPCO_2008, emd_Kopsinis_tsp_2009}, with 20 iterations, and the result is shown in \figref{fig:etea2_Example_1_x}b.
In this example, in order to perform decomposition, among the entire eight IMFs,
the thresholding is performed on IMFs 1-5, their summation is considered as the transients illustrated in \figref{fig:etea2_Example_1_x}a,
and IMFs 6-8 are considered as the estimation of the lowpass component.
The EMD based method achieves a decent denoising performance, but does not accurately estimate the components.
For instance, the simulated data has a smooth dip at about $ n = 700$ (circled in \figref{fig:etea2_Example_1_x}a).
EMD decomposes it into higher IMFs since they are varying slowly, which degrades the estimation.
We may group the lowpass and transient components differently to avoid this problem, 
for instance, grouping IMF 1-6 together in order to include more oscillations into transient component,
but this causes other distortion, 
where the decomposed transient component contains a lowpass signal and does not adhere to a baseline of zero.

The result obtained using second-order $\name$ is illustrated in \figref{fig:etea2_Example_1_x}b.
$\name$ estimates both the low-pass and transient components well, 
and recovers the signal by $x+f$ precisely with RMSE = 0.87 (with the smoothed $\ell_1$ penalty function).
The regularization parameter $\lam$ for problem \eqnref{eqn:etea_cost_R2} was similar to \eqnref{eqn:etea_three_sigma_h_1},
\begin{align}\label{eqn:etea_three_sigma_h_2}
	\lam \approx 2.5 \sigma_{w} \norm{ 2\h_2}_2,
\end{align}
where $\h_2$ is the impulse response of system $H^2(z)/ {R_2(z)}$, and $R_2(z)$ is defined in \eqnref{eqn:etea_R2z}.
The decomposition is accurate. 
There are no compensating waveforms between the estimated $x$ and $f$ at $ n = 700 $, comparing to the estimation in \figref{fig:etea2_Example_1_x}a.

Through numerical experiments, we found that second-order ETEA is not very sensitive to parameter $r$.
\figref{fig:etea2_Example_1_x}c shows the RMSE of denoising the data in \figref{fig:etea2_Example_1_data}a, using $ r \in [0.85, 0.99]$.
In this test, all filter parameters ($f_c$ and $d$) are the same and $\lam$ is set by \eqnref{eqn:etea_three_sigma_h_2}.
In most cases, the results are better than EMD-CIIT.
In addition, $r$ should not be too small or very close to 1.
If $r$ has to be very small to yield a good estimation of component $x$,
it must fluctuate extremely rapidly, then it must be closer to a sequence of  sparse spikes (i.e., Type 3 artifacts in \cite{artifacts_Islam_2014}), 
which differs from the signal model.
For such a signal, other algorithms are more suitable, e.g., LPF/CSD \cite{Selesnick_tara_2014, Selesnick_2013_lpftvd}.
Similarly, if $r$ has to be very close to 1 to fit the transients, then the transients must be very close to a piecewise linear signal,
which is also not how we model the signal initially.
As a consequence, we suggest to set $r$ in the range $0.90 < r < 0.98$.

\subsection{Example: ocular artifacts suppression}

%------------------------------------------------------------------------------%
\begin{figure}[t]
\centering
	\includegraphics[scale = \figurescale]{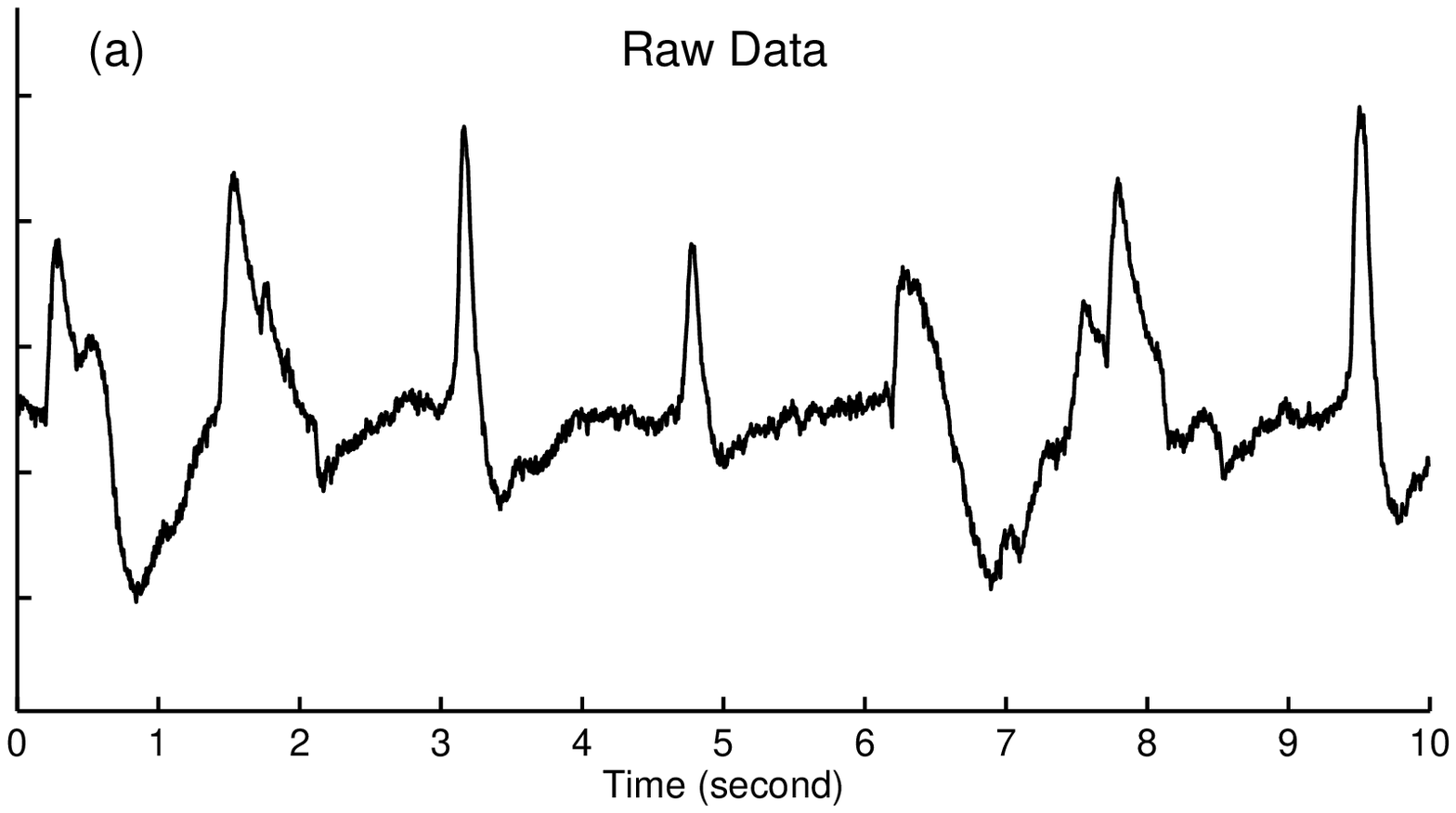} \\
	\includegraphics[scale = \figurescale]{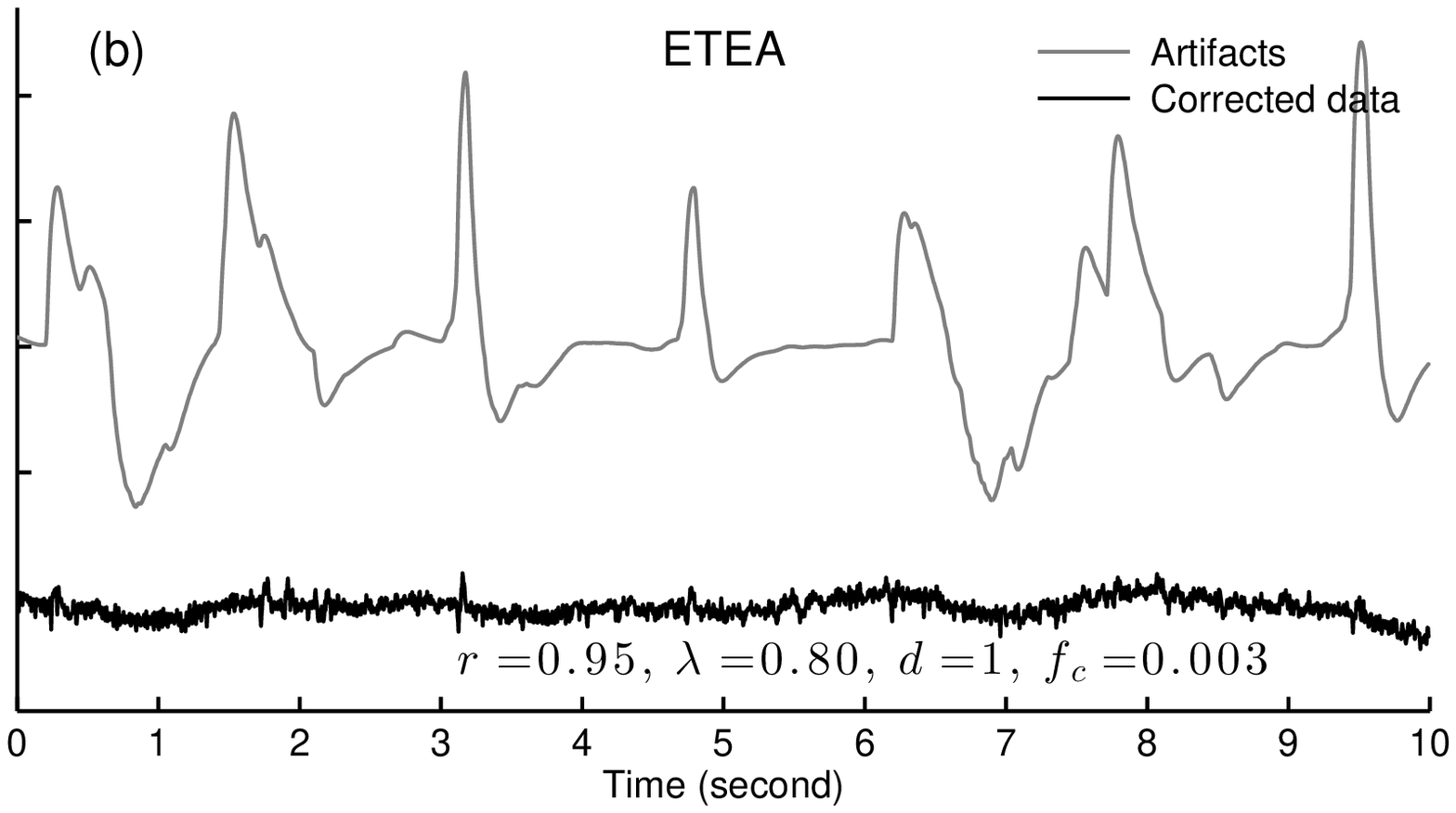} \\
	\includegraphics[scale = \figurescale]{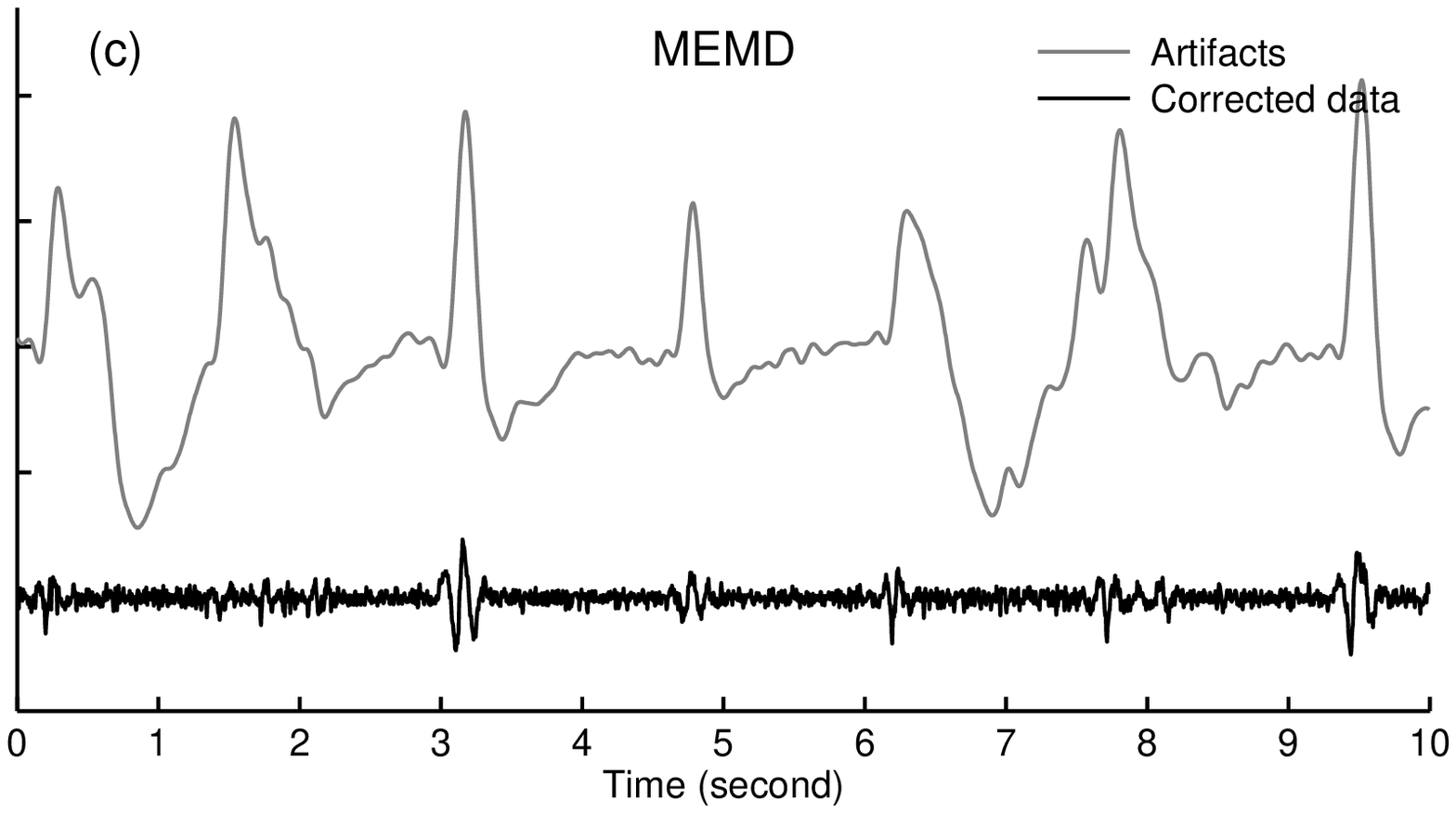}
	\caption{\textbf{Example 4:}~(a) EEG data with ocular artifacts, and corrected data using (b) $\name$ and (c) MEMD.}
	\label{fig:etea_Example_3}
\end{figure}
%------------------------------------------------------------------------------%

In this example, we use $\name$ with $\RR_2$ to correct EEG with eye blink/movement artifacts.
\figref{fig:etea_Example_3}a shows a 10 second signal from channel Fp1, with sampling rate $f_s = 256$ Hz, downloaded from \cite{eyeblink_Mandic_data}.
As a channel located on forehead, Fp1 is very sensitive to the motion of eyes and eyebrows,
and in this example, eye movement artifacts with large amplitudes are present through the entire signal.
Applying second-order $\name$ with $r=0.95$,  the results for corrected data and extracted OA are illustrated in \figref{fig:etea_Example_3}b.

As a comparison, we use multivariate empirical mode decomposition (MEMD) \cite{emd_Rehman_2010} to correct the data.
MEMD is a recently developed algorithm extending conventional EMD.
It has been used in different aspects of EEG signal analysis and applications \cite{emd_Chang_tnsre_2013, emd_Park_tnsre_2014, emd_Omidvarnia_2013},
including removing the ocular artifacts (OA) from multichannel EEG data \cite{emd_Rehman_2011, eyeblink_Molla_picassp_2012}.
In this example, we use 4 EEG channels (Fp1, Fp2, C3, C4) measured simultaneously as the input, 
and decompose the higher-index IMFs (low-frequency subbands) considered to be artifacts \cite[Section~V]{emd_Rehman_2011}.
More specifically, among all 14 IMFs decomposed in this example, we use IMF 1-4 as the corrected data, and the rest as artifacts.
The corrected data and estimated OA are shown in \figref{fig:etea_Example_3}c.

From the results in \figref{fig:etea_Example_3}b and \figref{fig:etea_Example_3}c, 
$\name$ estimates artifacts more clearly than MEMD.
In the MEMD result, some small-amplitude higher frequency oscillations leak into the artifact (e.g., about $t = 5.5$ and $9.0$ second).
Moreover, some artifacts are introduced after applying MEMD method to correct the data, (e.g., about $t = 3.0$ and $9.5$ second in \figref{fig:etea_Example_3}c).
Some oscillations are generated as transients where the abrupt artifacts occur.
In contrast, in \figref{fig:etea_Example_3}b,  there are no oscillations introduced either in the estimated artifacts or the corrected data.
\medskip

\noindent\textbf{Computational Efficiency.}
\figref{fig:etea_Example_time} shows the average computation time as a function of the signal length.
In this experiment, we calculate the time of computation of ETEA \eqnref{eqn:etea_cost_1} and second-order ETEA \eqnref{eqn:etea_cost_R2} 
with different filter settings (controlled by parameter $d$),
using input signal lengths from 5000 to $10^5$.
For each length, we average the computation time over 10 trials.
For each trial, run 40 iterations of each algorithm. 
The experiment is implemented in Matlab 8.4 on a MacBook~Pro 2012 with 2.7 GHz CPU.

As shown, the proposed algorithms have a run time of order $O(N)$.
Most of the computation time is consumed by the step of solving the linear system $\QQ^{-1}\b$ in Table~\ref{alg:etea},
where $\QQ$ is a banded matrix and we use a fast banded solver.

Additionally, fast iterative shrinkage-thresholding algorithm (FISTA) \cite{Beck_2009_SIAM, opt_Chambolle_fista_2014},
which is an acceleration scheme for iterative shrinkage-thresholding algorithm (ISTA) \cite{Chambolle_1998_tip}, 
(a special formulation of MM) may be used to further accelerate the algorithm.

%------------------------------------------------------------------------------%
\begin{figure}[!t]
\centering
	\includegraphics[scale = \figurescale] {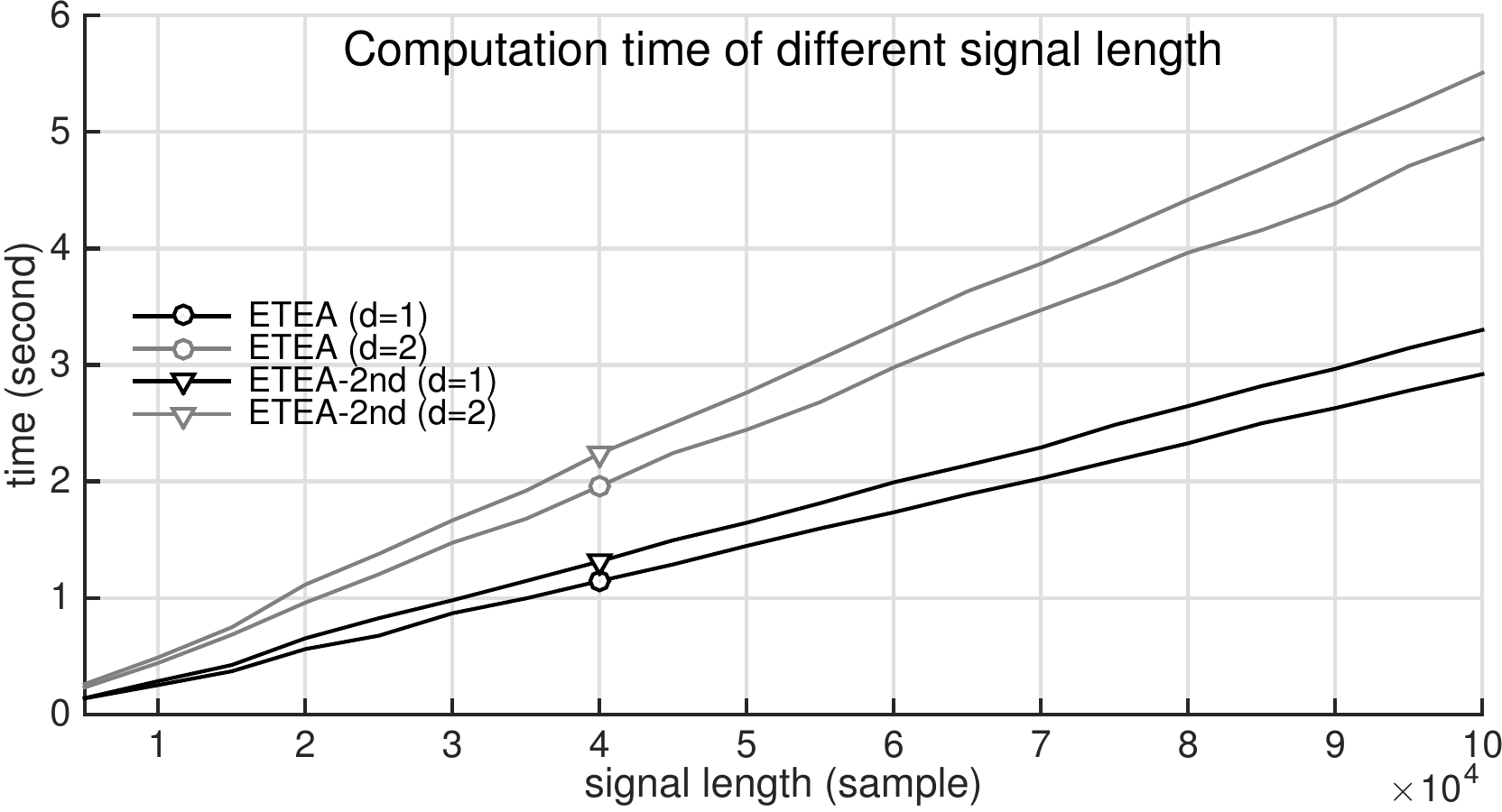} \\
	\caption{Comparison of computation time with different signal length.}
	\label{fig:etea_Example_time}
\end{figure}
%------------------------------------------------------------------------------%

%------------------------------------------------------------------------------%
\section{Conclusion}
%------------------------------------------------------------------------------%
This paper proposes a new algorithm for denoising and artifact removal for signals comprising artifacts arising in measured data, e.g., neural time-series recordings,
using sparse optimization.
The first algorithm, $\name$, assumes the signal is composed of a lowpass signal and an exponential transients (Type 1).
It is formulated as an optimization problem regularized by differentiable and smooth penalty function.
The second algorithm is an extension of $\name$, using a higher-order recursive filter, 
which is applicable for correction of continuous protuberance transients (Type 0), and more irregular artifacts.
As applications, we have shown that $\name$ with different regularizers ($\RR$ and $\RR_2$) 
are suitable for the suppression of Type~1 artifact in ECoG data
and ocular artifacts (OA) (as sequential Type~0 artifacts) in conventional EEG data,
with detailed comparisons to some state-of-the-art methods.
Both of the above algorithms are computationally efficient because they are formulated in terms of banded matrices.
A promising future work is to extend the above data correcting methods to multichannel data.

\section*{Acknowledgments}

The authors would like to thank 
Jonathan Viventi of the Department of Biomedical Engineering of Duke University,
and Justin Blanco of the Electrical and Computer Engineering Department of United States Naval Academy, 
for providing the data and giving useful comments.

\appendix
\section{Proof of Proposition~\ref{pro:etea_pro_1}} \label{app:etea_A}

\begin{proof}
Substitute the variables $u$ and $v$ in \eqnref{eqn:etea_majorizer} by
%------------------------------------------------------------------------------%
\begin{align}
	u = \sqrt{x^2+\eps}, \text{ and }  v = \sqrt{z^2+\eps}.
\end{align}
%------------------------------------------------------------------------------%
Therefore, for $x, z \in \mathbb{R}$, the inequality holds:
%------------------------------------------------------------------------------%
\begin{align} \label{eqn:etea_proof_variable_substitution}
	\frac{\phi'(\sqrt{z^2+\eps})}{2\sqrt{z^2+\eps}}(x^2+\eps) + \phi(\sqrt{z^2+\eps}) - \frac{\sqrt{z^2+\eps}}{2} \phi'(\sqrt{z^2+\eps})
\ge \phi(\sqrt{x^2+\eps}).
\end{align}
%------------------------------------------------------------------------------%
The right of the inequality is the majorizer of the smoothed penalty function $\phi(\sqrt{x^2+\eps})$.
\\
1)~If $z \neq 0$, we can multiply $z$ to the nominator and denominator of the first term on the left side of \eqnref{eqn:etea_proof_variable_substitution},
%------------------------------------------------------------------------------%
\begin{align} \label{eqn:etea_proof_rewrite_1}
	\frac{\phi'(\sqrt{z^2+\eps})}{2\sqrt{z^2+\eps}}(x^2+\eps) 							
	 = \frac{x^2}{2z}  \bigg( \phi'(\sqrt{z^2+\eps}) \frac{z}{\sqrt{z^2+\eps}} \bigg) 
	 +   \frac{\eps}{2z} \bigg( \phi'(\sqrt{z^2+\eps}) \frac{z}{\sqrt{z^2+\eps}} \bigg) 	
	 = 	\frac{x^2}{2z} \phie'(z) + \frac{\eps}{2z} \phie'(z).
\end{align}
%------------------------------------------------------------------------------%

Multiplying the nominator and denominator of the third term on the left side of \eqnref{eqn:etea_proof_variable_substitution} by $\sqrt{z^2+\eps}$, 
\begin{align} \label{eqn:etea_proof_rewrite_2}
	\frac{\sqrt{z^2+\eps}}{2} \phi'(\sqrt{z^2+\eps}) 
	=	& \frac{z^2}  {2\sqrt{z^2+\eps}} \phi'(\sqrt{z^2+\eps}) 		
	  +  \frac{\eps} {2\sqrt{z^2+\eps}} \phi'(\sqrt{z^2+\eps}) 							\nonumber \\[0.4em]
	=  &	\frac{z}{2}  \bigg( \phi'(\sqrt{z^2+\eps}) \frac{z}{\sqrt{z^2+\eps}} \bigg) 
	  +  \frac{\eps}{2z} \bigg( \phi'(\sqrt{z^2+\eps}) \frac{z}{\sqrt{z^2+\eps}} \bigg) \nonumber \\[0.4em]
	=  &\frac{z}{2} \phie'(z) + \frac{\eps}{2z} \phie'(z).
\end{align}
%------------------------------------------------------------------------------%
Using the results in \eqnref{eqn:etea_proof_rewrite_1} and \eqnref{eqn:etea_proof_rewrite_2},
the inequality \eqnref{eqn:etea_proof_variable_substitution} can be rewritten as
%------------------------------------------------------------------------------%
\begin{align}
	 \frac{x^2}{2z} \phie'(z) + \frac{\eps}{2z} \phie'(z)
	+\phie(z)
	-\frac{z}{2} \phie'(z) - \frac{\eps}{2z} \phie'(z)
	\ge \phie(x),
\end{align}
%------------------------------------------------------------------------------%
which can be reorganize into
%------------------------------------------------------------------------------%
\begin{align} 	\label{eqn:etea_proof_majorizer_2}
	g_{\eps}(x , z ; \phie) 
	& = \frac{\phi_{\eps}'(z)}{2z}x^2 
	+ \phi_{\eps}(z) - \frac{z}{2} \phi_{\eps}'(z) 
	\ge \phi_{\eps}(x).
\end{align}
%------------------------------------------------------------------------------%
\\
2)~If $z = 0$ and $x \neq 0$, by \textit{Lagrange's Mean Value Theorem} \cite[Theorem 5.10]{book_Rudin_1976},
since function $\phi(x)$ is continuous and $\phi''(x) \le 0$ on $x>0$, 
there exists a value $\xi$ in the range $\sqrt{\eps}<\xi<\sqrt{x^2+\eps}$ that
$\phi'(\sqrt{\eps}) \ge \phi'(\xi) \ge \phi'(\sqrt{x^2+\eps})$ satisfying
\begin{align} \label{eqn:etea_Lagrange_mean_3}
	\phi'(\xi) \left(\sqrt{x^2+\eps}-\sqrt{\eps}\right) 
	= \phi(\sqrt{x^2+\eps}) - \phi(\sqrt{\eps}).
\end{align}
Moreover, consider the square that is always positive
%------------------------------------------------------------------------------%
\begin{align}
	(\sqrt{x^2+\eps} - \sqrt{\eps}) ^2 = x^2 + \eps - 2  \sqrt{\eps (x^2+\eps)} +\eps > 0,
\end{align}
%------------------------------------------------------------------------------%
which implies the inequality
%------------------------------------------------------------------------------%
\begin{align} \label{eqn:etea_proof_inequality_1}
	x^2 > 2 \sqrt{\eps (x^2+\eps)} - 2 \eps.
\end{align}
%------------------------------------------------------------------------------%
Furthermore, we can multiply both sides of \eqref{eqn:etea_proof_inequality_1} by a positive term 
$\displaystyle \frac{\phi'(\sqrt{\eps}) } {2 \sqrt{\eps} }$,
and then adopt the result from \eqnref{eqn:etea_Lagrange_mean_3}, so that:
%------------------------------------------------------------------------------%
\begin{subequations}
\begin{align}
	\frac{\phi'(\sqrt{\eps}) } {2 \sqrt{\eps} }x^2  
	& > 
		\frac{\phi'(\sqrt{\eps}) } {2 \sqrt{\eps} } 
		\left(2 \sqrt{\eps (x^2+\eps)} - 2 \eps\right) \\
	& = \phi'(\sqrt{\eps}) \left(  \sqrt{x^2+\eps} - \sqrt{\eps} \right) \\
	& > \phi'(\xi) \left(  \sqrt{x^2+\eps} - \sqrt{\eps} \right)  \\
	& = \phi(\sqrt{x^2+\eps}) - \phi(\sqrt{\eps}),
\end{align}
\end{subequations}
%------------------------------------------------------------------------------%
which leads to
%------------------------------------------------------------------------------%
\begin{align}\label{eqn:etea_majorizer_zero}
	  \frac{\phi'(\sqrt{\eps}) } {2 \sqrt{\eps} }x^2 + \phi(\sqrt{\eps}) > \phi(\sqrt{x^2+\eps}) .
\end{align}
%------------------------------------------------------------------------------%

Because $\phie$ in \eqnref{eqn:etea_dphie} is differentiable on $\real$, we can find its second-order derivative at zero,
%------------------------------------------------------------------------------%
\begin{align}
	\lim_{z\to0} \phie''(z) = \frac{ \phi'(\sqrt{\eps}) }{\sqrt{\eps}},
\end{align}
%------------------------------------------------------------------------------%
and by L'H$\hat{\textrm{o}}$pital's rule \cite[Theorem 5.13]{book_Rudin_1976}, we have
%------------------------------------------------------------------------------%
\begin{align}
	\lim_{z\to0} \frac{\phie'(z)}{2z} = \lim_{z\to0} \frac{ \phie''(z)}{2} = \frac{ \phi'(\sqrt{\eps}) }{2\sqrt{\eps}},
\end{align}
%------------------------------------------------------------------------------%
which implies that \eqnref{eqn:etea_majorizer_zero} is in the same form of the majorizer \eqnref{eqn:etea_proof_majorizer_2} at $z=0$.
\\
3) If $x=z=0$, then the condition \eqnref{eqn:etea_mm_condition} follows immediately.
\end{proof}

%------------------------------------------------------------------------------%

%------------------------------------------------------------------------------%
% END
%------------------------------------------------------------------------------%

\begin{thebibliography}{10}

\bibitem{Akhtar_2012_SP}
M.~T. Akhtar, W.~Mitsuhashi, and C.~J. James.
\newblock Employing spatially constrained {ICA} and wavelet denoising, for
  automatic removal of artifacts from multichannel {EEG} data.
\newblock {\em Signal Processing}, 92(2):401--416, 2012.

\bibitem{Bach_2012_now}
F.~Bach, R.~Jenatton, J.~Mairal, and G.~Obozinski.
\newblock Optimization with sparsity-inducing penalties.
\newblock {\em Foundations and Trends in Machine Learning}, 4(1):1--106, 2012.

\bibitem{Beck_2009_SIAM}
A.~Beck and M.~Teboulle.
\newblock A fast iterative shrinkage-thresholding algorithm for linear inverse
  problems.
\newblock {\em SIAM J. Imag. Sci.}, 2(1):183--202, 2009.

\bibitem{Candes_2008_JFAP}
E.~J. Cand\`es, M.~B. Wakin, and S.~Boyd.
\newblock Enhancing sparsity by reweighted l1 minimization.
\newblock {\em J. {F}ourier Anal. Appl.}, 14(5):877--905, December 2008.

\bibitem{Cetin_tip_2001}
M.~Cetin and W.C. Karl.
\newblock Feature-enhanced synthetic aperture radar image formation based on
  nonquadratic regularization.
\newblock {\em IEEE Trans.\ Image Process.}, 10(4):623--631, April 2001.

\bibitem{Chambolle_1998_tip}
A.~Chambolle, R.~A. De~Vore, N.-Y. Lee, and B.~J. Lucier.
\newblock Nonlinear wavelet image processing: variational problems,
  compression, and noise removal through wavelet shrinkage.
\newblock {\em IEEE Trans.\ Image Process.}, 7(3):319--335, March 1998.

\bibitem{opt_Chambolle_fista_2014}
A.~Chambolle and V.~R. Dossal.
\newblock On the convergence of the iterates of {FISTA}.
\newblock {\em hal-01060130}, September 2014.
\newblock preprint.

\bibitem{emd_Chang_2009_iris}
C.-P. Chang, J.-C. Lee, Y.~Su, P.~S. Huang, and T.-M. Tu.
\newblock Using empirical mode decomposition for iris recognition.
\newblock {\em Computer Standards \& Interfaces}, 31(4):729--739, 2009.

\bibitem{emd_Chang_tnsre_2013}
H.-C. Chang, P.-L. Lee, M.-T. Lo, Y.-T. Wu, K.-W. Wang, and G.-Y. Lan.
\newblock Inter-trial analysis of post-movement beta activities in {EEG}
  signals using multivariate empirical mode decomposition.
\newblock {\em IEEE Trans.\ Neural Systems and Rehabilitation Engineering},
  21(4):607--615, July 2013.

\bibitem{Chen_Selesnick_2014_GSSD}
P.-Y. Chen and I.~W. Selesnick.
\newblock Group-sparse signal denoising: Non-convex regularization, convex
  optimization.
\newblock {\em IEEE Trans.\ Signal Process.}, 62(13):3464--3478, July 2014.

\bibitem{CD95}
R.~R. Coifman and D.~L. Donoho.
\newblock Translation-invariant de-noising.
\newblock In A.~Antoniadis, editor, {\em Wavelets and Statistics}.
  Springer-Verlag Lecture Notes, 1995.

\bibitem{eyeblink_Dammers_tbme_2008}
J.~Dammers, M.~Schiek, F.~Boers, C.~Silex, M.~Zvyagintsev, U.~Pietrzyk, and
  K.~Mathiak.
\newblock Integration of amplitude and phase statistics for complete artifact
  removal in independent components of neuromagnetic recordings.
\newblock {\em IEEE Trans.\ Biomed. Eng.}, 55(10):2353--2362, October 2008.

\bibitem{Durand_2001_ICASSP}
S.~Durand and J.~Froment.
\newblock Artifact free signal denoising with wavelets.
\newblock In {\em Proc.\ ICASSP}, 2001.

\bibitem{DurandFroment_2003_SIAM}
S.~Durand and J.~Froment.
\newblock Reconstruction of wavelet coefficients using total variation
  minimization.
\newblock {\em SIAM J. Sci. Comput.}, 24(5):1754--1767, 2003.

\bibitem{FBDN_2007_TIP}
M.~Figueiredo, J.~Bioucas-Dias, and R.~Nowak.
\newblock Majorization-minimization algorithms for wavelet-based image
  restoration.
\newblock {\em IEEE Trans.\ Image Process.}, 16(12):2980--2991, December 2007.

\bibitem{emd_Fleureau_2011}
J.~Fleureau, A.~Kachenoura, L.~Albera, J.-C. Nunes, and L.~Senhadji.
\newblock Multivariate empirical mode decomposition and application to
  multichannel filtering.
\newblock {\em Signal Processing}, 91(12):2783--2792, 2011.

\bibitem{Fuchs_2004_Tinfo}
J.-J. Fuchs.
\newblock On sparse representations in arbitrary redundant bases.
\newblock {\em IEEE Trans.\ Inform. Theory}, 50(6):1341--1344, 2004.

\bibitem{wave_Gao_1998}
H.~Gao.
\newblock Wavelet shrinkage denoising using the non-negative garrote.
\newblock {\em Journal of Computational and Graphical Statistics}, 7(4):pp.
  469--488, 1998.

\bibitem{Gholami_2013_SP}
A.~Gholami and S.~M. Hosseini.
\newblock A balanced combination of {T}ikhonov and total variation
  regularizations for reconstruction of piecewise-smooth signals.
\newblock {\em Signal Processing}, 93(7):1945--1960, 2013.

\bibitem{eyeblink_GuerreroMosquera_2012}
C.~Guerrero-Mosquera and A.~Navia-V{\'a}zquez.
\newblock Automatic removal of ocular artefacts using adaptive filtering and
  independent component analysis for electroencephalogram data.
\newblock {\em IET Signal Processing}, 6(2):99--106, 2012.

\bibitem{emd_Huang_2006_speech}
H.~Huang and J.~Pan.
\newblock Speech pitch determination based on {H}ilbert-{H}uang transform.
\newblock {\em Signal Processing}, 86(4):792--803, April 2006.

\bibitem{Huang_1998_EMD}
N.~E. Huang, Z.~Shen, S.~R. Long, M.~C. Wu, H.~H. Shih, Q.~Zheng, N.~C. Yen,
  C.~C. Tung, and H.~H. Liu.
\newblock The empirical mode decomposition and {H}ilbert spectrum for nonlinear
  and non-stationary time series analysis.
\newblock {\em Proc. Roy. Soc. Lon. A}, 454:903--995, March 1998.

\bibitem{emd_Huang_2003}
N.~E. Huang, M.-L.~C. Wu, S.~R. Long, S.~S.~P. Shen, W.~Qu, P.~Gloersen, and
  K.~L. Fan.
\newblock {A confidence limit for the empirical mode decomposition and Hilbert
  spectral analysis}.
\newblock {\em Proceedings of the Royal Society of London. Series A:
  Mathematical, Physical and Engineering Sciences}, 459(2037):2317--2345,
  September 2003.

\bibitem{mm_Hunter_tutorial_2004}
D.~R. Hunter and K.~Lange.
\newblock A tutorial on {MM} algorithms.
\newblock {\em Amer. Statist.}, 58:30--37, 2004.

\bibitem{artifacts_Islam_2014}
M.~K. Islam, A.~Rastegarnia, A.~T. Nguyen, and Z.~Yang.
\newblock Artifact characterization and removal for in vivo neural recording.
\newblock {\em Journal of Neuroscience Methods}, 226(0):110--123, 2014.

\bibitem{eyeblink_Joyce_2004}
C.~A. Joyce, I.~F. Gorodnitsky, and M.~Kutas.
\newblock Automatic removal of eye movement and blink artifacts from {EEG} data
  using blind component separation.
\newblock {\em Psychophysiology}, 41(2):313--325, 2004.

\bibitem{tv_Karahanoglu_2011}
F.~I. Karahanoglu, I.~Bayram, and D.~Van De~Ville.
\newblock A signal processing approach to generalized 1-d total variation.
\newblock {\em IEEE Trans.\ Signal Process.}, 59(11):5265--5274, November 2011.

\bibitem{ban_Kilic_2008}
E.~Kilic.
\newblock Explicit formula for the inverse of a tridiagonal matrix by backward
  continued fractions.
\newblock {\em Applied Mathematics and Computation}, 197(1):345--357, 2008.

\bibitem{ban_Kilic_2013}
E.~Kilic and P.~Stanica.
\newblock The inverse of banded matrices.
\newblock {\em Journal of Computational and Applied Mathematics},
  237(1):126--135, 2013.

\bibitem{l1_Kim_2009}
S.~Kim, K.~Koh, S.~Boyd, and D.~Gorinevsky.
\newblock $\ell_1$ trend filtering.
\newblock {\em SIAM Review}, 51(2):339--360, 2009.

\bibitem{emd_Kopsinis_EUSIPCO_2008}
Y.~Kopsinis and S.~McLaughlin.
\newblock Empirical mode decomposition based soft-thresholding.
\newblock In {\em 16th European Signal Processing Conference}, 2008.

\bibitem{emd_Kopsinis_tsp_2009}
Y.~Kopsinis and S.~McLaughlin.
\newblock Development of {EMD}-based denoising methods inspired by wavelet
  thresholding.
\newblock {\em IEEE Trans.\ Signal Process.}, 57(4):1351--1362, April 2009.

\bibitem{Kowalski_2009}
M.~Kowalski.
\newblock Sparse regression using mixed norms.
\newblock {\em Applied and Computational Harmonic Analysis}, 27(3):303 -- 324,
  2009.

\bibitem{Kozlov_2010_geomath}
I.~Kozlov and A.~Petukhov.
\newblock Sparse solutions of underdetermined linear systems.
\newblock In W.~Freeden et~al., editor, {\em Handbook of Geomathematics}.
  Springer, 2010.

\bibitem{opt_Lange_book_2004}
K.~Lange.
\newblock {\em Optimization}.
\newblock Springer New York, 2004.

\bibitem{mm_Lange_2000}
K.~Lange, D.~Hunter, and I.~Yang.
\newblock Optimization transfer using surrogate objective functions.
\newblock {\em J. of Comp. Graph. Statist.}, 9:1--20, 2000.

\bibitem{Mammone_2012}
N.~Mammone, F.~{La Foresta}, and F.~C. Morabito.
\newblock Automatic artifact rejection from multichannel scalp {EEG} by wavelet
  {ICA}.
\newblock {\em IEEE J. Sensors}, 12(3):533--542, March 2012.

\bibitem{eyeblink_Mandic_data}
D.~Mandic.
\newblock Empirical mode decomposition, multivariate {EMD}, matlab and data
  sources.

\bibitem{emd_Mandic_2013}
D.~P. Mandic, N.~U. Rehman, Z.~Wu, and N.~E. Huang.
\newblock Empirical mode decomposition-based time-frequency analysis of
  multivariate signals: The power of adaptive data analysis.
\newblock {\em Signal Processing Magazine, IEEE}, 30(6):74--86, November 2013.

\bibitem{Molavi_2012}
B.~Molavi and G.~A. Dumont.
\newblock Wavelet-based motion artifact removal for functional near-infrared
  spectroscopy.
\newblock {\em Physiological Measurement}, 33(2):259, 2012.

\bibitem{Molla_2012}
M.~K.~I. Molla, M.~R. Islam, T.~Tanaka, and T.~M. Rutkowski.
\newblock Artifact suppression from {EEG} signals using data adaptive time
  domain filtering.
\newblock {\em Neurocomputing}, 97:297--308, 2012.

\bibitem{eyeblink_Molla_picassp_2012}
M.~K.~I. Molla, T.~Tanaka, and T.~M. Rutkowski.
\newblock Multivariate {EMD} based approach to {EOG} artifacts separation from
  {EEG}.
\newblock In {\em Proc.\ ICASSP 2012}, pages 653--656, 2012.

\bibitem{eyeblink_Molla_picassp_2010}
M.~K.~I. Molla, T.~Tanaka, T.~M. Rutkowski, and A.~Cichocki.
\newblock Separation of {EOG} artifacts from {EEG} signals using bivariate
  {EMD}.
\newblock In {\em Proc.\ ICASSP 2010}, pages 562--565, 2010.

\bibitem{eyeblink_Nazarpour_tbme_2008}
K.~Nazarpour, Y.~Wongsawat, S.~Sanei, J.~A. Chambers, and S.~Oraintara.
\newblock Removal of the eye-blink artifacts from {EEG}s via {STF-TS} modeling
  and robust minimum variance beamforming.
\newblock {\em IEEE Trans.\ Biomed. Eng.}, 55(9):2221--2231, 2008.

\bibitem{Needell_reweight_2009}
D.~Needell.
\newblock Noisy signal recovery via iterative reweighted l1-minimization.
\newblock In {\em Proc.~Forty-Third Asilomar Conference on Signals, Systems and
  Computers}, pages 113--117, 2009.

\bibitem{smooth_Nikolova_2005_MMS}
M.~Nikolova.
\newblock Analysis of the recovery of edges in images and signals by minimizing
  nonconvex regularized least-squares.
\newblock {\em Multiscale Modeling and Simulation}, 4(3):960--991, 2005.

\bibitem{Ning_QRS_2013}
X.~Ning and I.~W. Selesnick.
\newblock {ECG} enhancement and {QRS} detection based on sparse derivatives.
\newblock {\em Biomedical Signal Processing and Control}, 8(6):713--723, 2013.

\bibitem{eyeblink_Noureddin_tbme_2012}
B.~Noureddin, P.~D. Lawrence, and G.~E. Birch.
\newblock Online removal of eye movement and blink {EEG} artifacts using a
  high-speed eye tracker.
\newblock {\em IEEE Trans.\ Biomed. Eng.}, 59(8):2103--2110, 2012.

\bibitem{emd_Omidvarnia_2013}
A.~Omidvarnia, G.~Azemi, P.~B. Colditz, and B.~Boashash.
\newblock A time-frequency based approach for generalized phase synchrony
  assessment in nonstationary multivariate signals.
\newblock {\em Digital Signal Processing}, 23(3):780--790, 2013.

\bibitem{emd_Park_tnsre_2014}
C.~Park, M.~Plank, J.~Snider, S.~Kim, H.~C. Huang, S.~Gepshtein, T.~P. Coleman,
  and H.~Poizner.
\newblock {EEG} gamma band oscillations differentiate the planning of spatially
  directed movements of the arm versus eye: Multivariate empirical mode
  decomposition analysis.
\newblock {\em IEEE Trans.\ Neural Systems and Rehabilitation Engineering},
  22(5):1083--1096, September 2014.

\bibitem{Rangayyan_book}
R.~M. Rangayyan.
\newblock {\em Biomedical Signal Analysis - A Case-study Approach}.
\newblock IEEE and Wiley, New York, NY, 2002.

\bibitem{emd_Rehman_2010}
N.~U. Rehman and D.~P. Mandic.
\newblock Multivariate empirical mode decomposition.
\newblock {\em Proceedings of the Royal Society A}, 466(2117):1291--1302, 2010.

\bibitem{emd_Rehman_2011}
N.~U. Rehman and D.~P. Mandic.
\newblock Filter bank property of multivariate empirical mode decomposition.
\newblock {\em IEEE Trans.\ Signal Process.}, 59(5):2421--2426, May 2011.

\bibitem{Rilling_2008_TSP}
G.~Rilling and P.~Flandrin.
\newblock One or two frequencies? {T}he empirical mode decomposition answers.
\newblock {\em IEEE Trans.\ Signal Process.}, 56(1):85--95, 2008.

\bibitem{book_Rudin_1976}
W.~Rudin.
\newblock {\em Principles of mathematical analysis.}
\newblock McGraw-Hill, 1976.

\bibitem{Sato_2006_NeuroImage}
H.~Sato, N.~Tanaka, M.~Uchida, Y.~Hirabayashi, M.~Kanai, T.~Ashida, I.~Konishi,
  and A.~Maki.
\newblock Wavelet analysis for detecting body-movement artifacts in optical
  topography signals.
\newblock {\em NeuroImage}, 33(2):580--587, 2006.

\bibitem{mm_Schifano_2010}
E.~D. Schifano, R.~L. Strawderman, and M.~T. Wells.
\newblock Majorization-minimization algorithms for nonsmoothly penalized
  objective functions.
\newblock {\em Electron. J. Statist.}, 4:1258--1299, 2010.

\bibitem{Selesnick_2012_polynomial}
I.~W. Selesnick, S.~Arnold, and V.~Dantham.
\newblock Polynomial smoothing of time series with additive step
  discontinuities.
\newblock {\em IEEE Trans.\ Signal Process.}, 60(12):6305--6318, December 2012.

\bibitem{Selesnick_2013_max_sparse}
I.~W. Selesnick and I.~Bayram.
\newblock Sparse signal estimation by maximally sparse convex optimization.
\newblock {\em IEEE Trans.\ Signal Process.}, 62(5):1078--1092, March 2014.

\bibitem{Selesnick_tara_2014}
I.~W. Selesnick, H.~L. Graber, Y.~Ding, T~Zhang, and R.~L. Barbour.
\newblock Transient artifact reduction algorithm ({TARA}) based on sparse
  optimization.
\newblock {\em IEEE Trans.\ Signal Process.}, 62(24):6596--6611, December 2014.

\bibitem{Selesnick_2013_lpftvd}
I.~W. Selesnick, H.~L. Graber, S.~Douglas, S.~Pfeil, and R.~L. Barbour.
\newblock Simultaneous low-pass filtering and total variation denoising.
\newblock {\em IEEE Trans.\ Signal Process.}, 62(5):1109--1124, March 2014.

\bibitem{emd_Tang_sp_2012}
B.~Tang, S.~Dong, and T.~Song.
\newblock Method for eliminating mode mixing of empirical mode decomposition
  based on the revised blind source separation.
\newblock {\em Signal Processing}, 92(1):248--258, 2012.

\bibitem{Wipf_2010_TSP}
D.~Wipf and S.~Nagarajan.
\newblock Iterative reweighted $\ell_1$ and $\ell_2$ methods for finding sparse
  solutions.
\newblock {\em IEEE. J. Sel. Top. Signal Processing}, 4(2):317--329, April
  2010.

\bibitem{eyeblink_Wongsawat_2008}
Y.~Wongsawat.
\newblock Efficient implementation of {RMVB} for eyeblink artifacts removal of
  {EEG} via {STF-TS} modeling.
\newblock In {\em Proc. ROBIO 2008}, pages 1567--1572, 2008.

\bibitem{emd_Xie_sp_2014_face}
X.~Xie.
\newblock Illumination preprocessing for face images based on empirical mode
  decomposition.
\newblock {\em Signal Processing}, 103(0):250--257, 2014.

\bibitem{Yaghoobi_2009_TSP}
M.~Yaghoobi, T.~Blumensath, and M.~E. Davies.
\newblock Dictionary learning for sparse approximations with the majorization
  method.
\newblock {\em IEEE Trans.\ Signal Process.}, 57(6):2178--2191, June 2009.

\bibitem{emd_Yan_2014_ECG}
J.~Yan and L.~Lu.
\newblock Improved {H}ilbert{\textendash}{H}uang transform based weak signal
  detection methodology and its application on incipient fault diagnosis and
  {ECG} signal analysis.
\newblock {\em Signal Processing}, 98:74--87, 2014.

\bibitem{Yu_sp_2012}
S.~Yu, A.~S. Khwaja, and J.~Ma.
\newblock Compressed sensing of complex-valued data.
\newblock {\em Signal Processing}, 92(2):357--362, 2012.

\bibitem{Zeng_2013_Sensors}
H.~Zeng, A.~Song, R.~Yan, and H.~Qin.
\newblock {EOG} artifact correction from {EEG} recording using stationary
  subspace analysis and empirical mode decomposition.
\newblock {\em Sensors}, 13(11):14839--14859, 2013.

\end{thebibliography}
\end{document}